\documentclass[pdflatex,sn-nature]{sn-jnl}
\usepackage{url}

\usepackage{graphicx}%
\usepackage{multirow}%
\usepackage{amsmath,amssymb,amsfonts}%
\usepackage{amsthm}%
\usepackage{mathrsfs}%
\usepackage[title]{appendix}%
\usepackage{xcolor}%
\usepackage{textcomp}%
\usepackage{manyfoot}%
\usepackage{booktabs}%
\usepackage{algorithm}%
\usepackage{algorithmicx}%
\usepackage{algpseudocode}%
\usepackage{listings}%
\usepackage{hyperref}
\usepackage{url}
\usepackage{nicefrac}       
\usepackage{microtype}      
\usepackage[dvipsnames]{xcolor}       
\usepackage{colortbl} 
\usepackage{subcaption}
\usepackage{wrapfig}

\usepackage{siunitx}
\DeclareSIUnit{\cells}{cells}

\usepackage{import}

\usepackage{listings}
\usepackage{xcolor}

\definecolor{codegreen}{rgb}{0,0.6,0}
\definecolor{codegray}{rgb}{0.5,0.5,0.5}
\definecolor{codepurple}{rgb}{0.58,0,0.82}
\definecolor{backcolour}{rgb}{0.95,0.95,0.92}

\lstdefinestyle{mystyle}{
    backgroundcolor=\color{backcolour},   
    commentstyle=\color{codegreen},
    keywordstyle=\color{magenta},
    numberstyle=\tiny\color{codegray},
    stringstyle=\color{codepurple},
    basicstyle=\ttfamily\footnotesize,
    breakatwhitespace=false,         
    breaklines=true,                 
    captionpos=b,                    
    keepspaces=true,                 
    numbers=left,                    
    numbersep=5pt,                  
    showspaces=false,                
    showstringspaces=false,
    showtabs=false,                  
    tabsize=2
}

\usepackage{comment}

\usepackage{todonotes}

\usepackage[superscript,biblabel]{cite}

\theoremstyle{thmstyleone}%
\newtheorem{theorem}{Theorem}
\newtheorem{lemma}[theorem]{Lemma}

\theoremstyle{thmstyletwo}%

\theoremstyle{thmstylethree}%
\newtheorem{corollary}{Corollary}

\newtheorem{assumption}{Assumption}

\begin{document}

\title[FedMAP]{FedMAP: Personalised Federated Learning for Real Large-Scale Healthcare Systems}


\author[1]{\fnm{Fan} \sur{Zhang}} 
\author[1]{\fnm{Daniel} \sur{Kreuter}} 
\author[12]{\fnm{Carlos} \sur{Esteve-Yagüe}} 
\author[1]{\fnm{Sören} \sur{Dittmer}} 
\author[10]{\fnm{Javier} \sur{Fernandez-Marques}}
\author[3,5,13,14]{\fnm{Samantha} \sur{Ip}}
\author[2]{\fnm{BloodCounts!} \sur{Consortium}}
\author[21]{\fnm{Norbert C.J.} \sur{de Wit}}
\author[3,4,5,6,13,15,16]{\fnm{Angela} \sur{Wood}}
\author[9]{\fnm{James HF Rudd}} 
\author[10,11]{\fnm{Nicholas} \sur{Lane}}
\author[3,7,8]{\fnm{Nicholas} \spfx{S} \sur{Gleadall}}
\author[1]{\fnm{Carola-Bibiane} \sur{Schönlieb}} 
\author*[1,9]{\fnm{Michael} \sur{Roberts}} \email{mr808@cam.ac.uk}

\affil[1]{\orgdiv{Department of Applied Mathematics and Theoretical Physics}, 
           \orgname{University of Cambridge}, 
           \orgaddress{\city{Cambridge}, \country{UK}}}

\affil[2]{\orgdiv{A list of authors and their affiliations appears at the end of the paper}}

\affil[3]{\orgdiv{Victor Phillip Dahdaleh Heart and Lung Research Institute}, 
           \orgname{University of Cambridge}, 
           \orgaddress{\street{Cambridge Biomedical Campus}, \city{Cambridge}, \country{UK}}}

\affil[4]{\orgdiv{British Heart Foundation Data Science Centre}, 
           \orgname{Health Data Research UK}, 
           \orgaddress{\city{London}, \country{UK}}}

\affil[5]{\orgdiv{Department of Public Health and Primary Care, British Heart Foundation Cardiovascular Epidemiology Unit}, 
           \orgname{University of Cambridge}, 
           \orgaddress{\city{Cambridge}, \country{UK}}}

\affil[6]{\orgdiv{Health Data Research UK Cambridge}, 
           \orgaddress{\street{Wellcome Genome Campus and University of Cambridge}, 
                        \city{Cambridge}, \country{UK}}}

\affil[7]{\orgdiv{Department of Haematology}, 
           \orgname{University of Cambridge}, 
           \orgaddress{\street{Cambridge Biomedical Campus}, \city{Cambridge}, \country{UK}}}

\affil[8]{\orgname{NHS Blood and Transplant}, 
           \orgaddress{\street{Cambridge Biomedical Campus}, \city{Cambridge}, \country{UK}}}

\affil[9]{\orgdiv{Department of Medicine}, 
           \orgname{University of Cambridge}, 
           \orgaddress{\street{Cambridge Biomedical Campus}, \city{Cambridge}, \country{UK}}}

\affil[10]{\orgname{Flower Labs}, 
            \orgaddress{\city{Cambridge}, \country{UK}}}

\affil[11]{\orgdiv{Department of Computer Science and Technology}, 
            \orgname{University of Cambridge}, 
            \orgaddress{\city{Cambridge}, \country{UK}}}

\affil[12]{\orgdiv{Department of Mathematics}, 
            \orgname{University of Alicante}, 
            \orgaddress{\city{Alicante}, \country{Spain}}}

\affil[13]{\orgdiv{The Cambridge Centre for AI in Medicine, Department of Applied Mathematics and Theoretical Physics}, 
            \orgname{University of Cambridge}, 
            \orgaddress{\city{Cambridge}, \country{UK}}}

\affil[14]{\orgdiv{Centre for Cancer Genetic Epidemiology, Department of Public Health and Primary Care}, 
            \orgname{University of Cambridge}, 
            \orgaddress{\city{Cambridge}, \country{UK}}}

\affil[15]{\orgdiv{British Heart Foundation Centre of Research Excellence}, 
            \orgname{University of Cambridge}, 
            \orgaddress{\city{Cambridge}, \country{UK}}}

\affil[16]{\orgdiv{National Institute for Health and Care Research Blood and Transplant Research Unit in Donor Health and Behaviour}, 
            \orgname{University of Cambridge}, 
            \orgaddress{\city{Cambridge}, \country{UK}}}

\abstract{
Federated learning (FL) promises to enable collaborative machine learning across healthcare sites whilst preserving data privacy. Practical deployment remains limited by statistical heterogeneity arising from differences in patient demographics, treatments, and outcomes, and infrastructure constraints. We introduce FedMAP, a personalised FL (PFL) framework that addresses heterogeneity through local Maximum a Posteriori (MAP) estimation with Input Convex Neural Network priors. These priors represent global knowledge gathered from other sites that guides the model while adapting to local data, and we provide a formal proof of convergence. Unlike many PFL methods that rely on fixed regularisation, FedMAP's prior adaptively learns patterns that capture complex inter-site relationships. We demonstrate improved performance compared to local training, FedAvg, and several PFL methods across three large-scale clinical datasets: 10-year cardiovascular risk prediction (CPRD, \num{387} general practitioner practices, \num{258688} patients), iron deficiency detection (INTERVAL, \num{4} donor centres, \num{31949} blood donors), and mortality prediction (eICU, \num{150} hospitals, \num{44842} patients). FedMAP incorporates a three-tier design that enables participation across healthcare sites with varying infrastructure and technical capabilities, from full federated training to inference-only deployment. Geographical analysis reveals substantial equity improvements, with underperforming regions achieving up to \SI{14.3}{\percent} performance gains. This framework provides the first practical pathway for large-scale healthcare FL deployment, which ensures clinical sites at all scales can benefit, equity is enhanced, and privacy is retained.
}


\maketitle

\section*{Introduction}\label{sec:intro}
Federated learning (FL) is a decentralised artificial intelligence (AI) paradigm that enables multiple sites to collaboratively train a shared model whilst maintaining data privacy by keeping datasets locally stored and only exchanging model parameters or gradients rather than raw data. In healthcare settings, leveraging data from multiple hospitals and clinics, FL could enable machine learning (ML) models to capture variation in patient demographics, comorbidities and treatment patterns, with the potential to provide more robust and generalisable insights for diagnosis and prognosis~\cite{dayan_federated_2021}. The appeal of FL lies in its potential to address two fundamental challenges for medical AI: First, clinical data are highly regulated: raw electronic health records and laboratory results cannot be pooled centrally due to stringent privacy legislation, security risks, and the prohibitive cost of transferring large volumes of sensitive information~\cite{gdpr2016}. Second, models trained in one healthcare setting often fail to generalise to others, as hospitals serve very different populations, including outpatient, trauma, paediatric, and oncology populations, and they differ in demographics, clinical pathways, and the prevalence of outcomes. These constraints are particularly acute in medicine, where legal and ethical obligations mandate that patient-level data remain within site boundaries, often managed through Trusted Research Environments (TREs)~\cite{bradford_international_2020}. Our recent systematic review of FL methods in healthcare identified that despite its promise, FL remains predominantly confined to proof of concept studies \cite{zhang2023recent}, with only \SI{5.2}{\percent} of reported implementations demonstrating real world clinical use~\cite{teo_federated_2024}. This limited deployment reflects significant practical challenges: methodological constraints in adapting FL to medical data, and substantial disparities in computational resources and technical expertise across healthcare sites~\cite{rieke_future_2020}. These challenges can be grouped into three categories: (i) statistical heterogeneity in clinical data, (ii) methodological limitations of the current FL 
approaches and (iii) infrastructure disparities across healthcare sites. We describe each of these in the following sections.\\

\noindent
First, real world clinical data are commonly non-independent and identically distributed (non-IID) across sites due to hospitals serving different patient populations, using different diagnostic tools, and following different clinical protocols, which introduce label distribution skew, feature distribution skew and quantity skew simultaneously~\cite{tongDistributedLearningHeterogeneous2022}.
Early FL work demonstrated feasibility in radiology~\cite{sheller_federated_2020}, genomics~\cite{kolobkov_efficacy_2024} and electronic health records~\cite{dayan_federated_2021}, but mostly used Federated Averaging (FedAvg)~\cite{mcmahan_communication-efficient_2016} which assumes IID data. This mismatch between algorithmic assumptions and healthcare data reality results in reduced generalisability, slower convergence, lower accuracy and higher communication costs, collectively failing to adequately serve individual sites' needs~\cite{li_convergence_2020}.\\

\noindent
Second, popular FL methods such as FedAvg~\cite{mcmahan_communication-efficient_2016}, FedProx~\cite{li2020federated}, and SCAFFOLD~\cite{karimireddy2020scaffold} are based on maximum likelihood estimation (MLE) for the approximation of optimal model parameters. These methods combine models trained at individual sites to produce a single global model. However, this fails to achieve optimal local performance for non-IID data, as a single global objective function cannot adequately capture the heterogeneity in data across different sites~\cite{li_fedbn_2021}. Local gradients from different sites often diverge significantly, resulting in slow or suboptimal convergence~\cite{li_convergence_2020}. Existing FedAvg variants attempt to address these limitations through various mechanisms: FedProx~\cite{li2020federated} employs fixed quadratic regularisation to constrain deviation from the global model, yet lacks the adaptivity required for complex inter-site relationships, whilst q-FedAvg~\cite{li_fair_2020} reweights site contributions based on loss functions but remains vulnerable to overweighting noisy or highly skewed clinical datasets~\cite{tang_virtual_2022}.\\

\noindent
Apart from classic FedAvg based methods, recent research highlights that personalised FL (PFL), where each site obtains a model tailored to their specific data distribution, can mitigate the impact of data heterogeneity \cite{9210355}. The field spans diverse methodological approaches from meta-learning~\cite{NEURIPS2020_243892103559bfe} and client clustering~\cite{REN2023119499} to Bayesian methods~\cite{cao_bayesian_2023}. 
Among these approaches, PFL methods adding regularisation terms to control deviation between local and global models, remain especially appealing in healthcare context, as they allow for the model architecture and task flexibility needed for healthcare ML applications. Several popular methods, such as FedDyn~\cite{acar_federated_2021}, Ditto~\cite{pmlr-v139-li21h}, and pFedME~\cite{t2020personalized}, share these advantages but rely on fixed regulariser functions, such as proximal terms or gradient alignment, which limit their flexibility to adapt to varying parameter distributions across sites. Similarly, Bayesian approaches such as pFedBayes~\cite{pmlr-v162-zhang22o} and pFedGP~\cite{achituve2021personalized} face constraints in their underlying regularisation structures due to rigid prior assumptions that limit regularisation flexibility.\\

\noindent
Finally, existing FL implementations are not feasible for many hospitals, since infrastructure and technical expertise requirements are high. Beyond access to high-performance servers and GPUs \cite{kaissis_secure_2020}, hospitals must also maintain secure infrastructure, implement robust data governance processes, and ensure appropriate anonymisation and informatics support. These requirements create substantial barriers for sites with limited technical staff or resources. Academic medical centres and research hospitals may have high-performance computing clusters, standardised processes, and dedicated data science and bioinformatics research teams that can implement and maintain FL systems. By contrast, many regional hospitals operate with more moderate computational resources and established electronic health record systems, but often have limited ML expertise and technical support staff. These constraints make it difficult for them to participate in complex FL networks that require substantial local model training, infrastructure setup, and ongoing maintenance~\cite{sheller_federated_2020}. Community hospitals, which typically handle initial patient assessments and referrals rather than complete care pathways, face the most significant resource constraints for implementing sophisticated FL infrastructure~\cite{hsu_clinical_2021}.\\

\noindent
To address these limitations, we introduce FedMAP, a new PFL framework that is derived from first principles of Bayesian inference to address the FL challenges of non-IID data by utilising a global prior distribution derived from MAP. Building on the MAP estimation framework, we formulate the FL problem as a bi-level optimisation~\cite{colson_overview_2007} that adaptively learns and updates a global ``prior'' model, parameterised using Input Convex Neural Networks (ICNNs)~\cite{pmlr-v70-amos17b}, regularising the local model optimisation. By using ICNNs to parameterise the global regularisation function, we obtain a convex MAP objective, which in turn allows us to prove existence, uniqueness, and convergence in our bi-level optimisation framework. Additionally, unlike existing regularisation based PFL methods such as FedDyn, Ditto, and pFedME that use fixed penalties, ICNNs allow for more effective adaptation to heterogeneous site distributions while maintaining convergence guarantees. This approach effectively balances knowledge sharing and local data distribution adaptation. Moreover, the learnable ICNN-based regularisation component of FedMAP can be decoupled and reused independently of the federated training process, which enables deployment to be adapted for sites with different computational resources. 
In particular, sites with extensive ML infrastructure, dedicated IT support staff, and expertise to maintain complex FL systems can participate in full federated training. Sites with adequate local computing resources and basic technical support but limited networking capabilities or specialised ML expertise can perform local fine-tuning using pre-trained models. Sites with minimal computational infrastructure or technical support are accommodated through inference-only deployment. We refer to these as Tier 1 (T1), Tier 2 (T2), and Tier 3 (T3) respectively (Figure~\ref{fig:fedmap}). T1 sites can take part in FL training using FedMAP; the combination of learned global prior and global model parameters enables T2 sites to perform FedMAP local optimisation on local data without additional communication and the global model enables T3 sites to perform inference on local data without training. This tiered approach directly addresses one of the key deployment challenges identified in our systematic review~\cite{zhang2023recent}, which concluded that the ``lack of widespread computational capability remains a barrier to the mass adoption of FL in real-world healthcare settings''.\\

\noindent
Our contributions include four advances in healthcare FL, with a focus on addressing real clinical data challenges and enabling practical deployment. First, we introduce FedMAP, a provably convergent PFL framework that uses MAP estimation with learnable ICNN priors. Second, we demonstrate FedMAP outperforms existing FL and PFL methods on three large-scale clinical datasets, including the first FL applications to CPRD~\cite{herrett_data_2015, xu_prediction_2021} for cardiovascular risk prediction and INTERVAL~\cite{angelantonio_efficiency_2017} for blood donor iron deficiency detection, showing computational scalability for healthcare federations exceeding 300 healthcare sites and \num{300000} patient records. Geographical analysis reveals that underperforming regions achieve up to 14.3\% performance gains, which show FedMAP's potential to reduce healthcare disparities. Third, we propose a flexible deployment strategy through a three-tier architecture accommodating diverse healthcare infrastructure capabilities. Fourth, we integrate FedMAP code with the popular Flower FL framework~\cite{beutel2020flower} for direct adoption by the community.



\begin{figure}[h]
\centering
\includegraphics[width=1\textwidth]{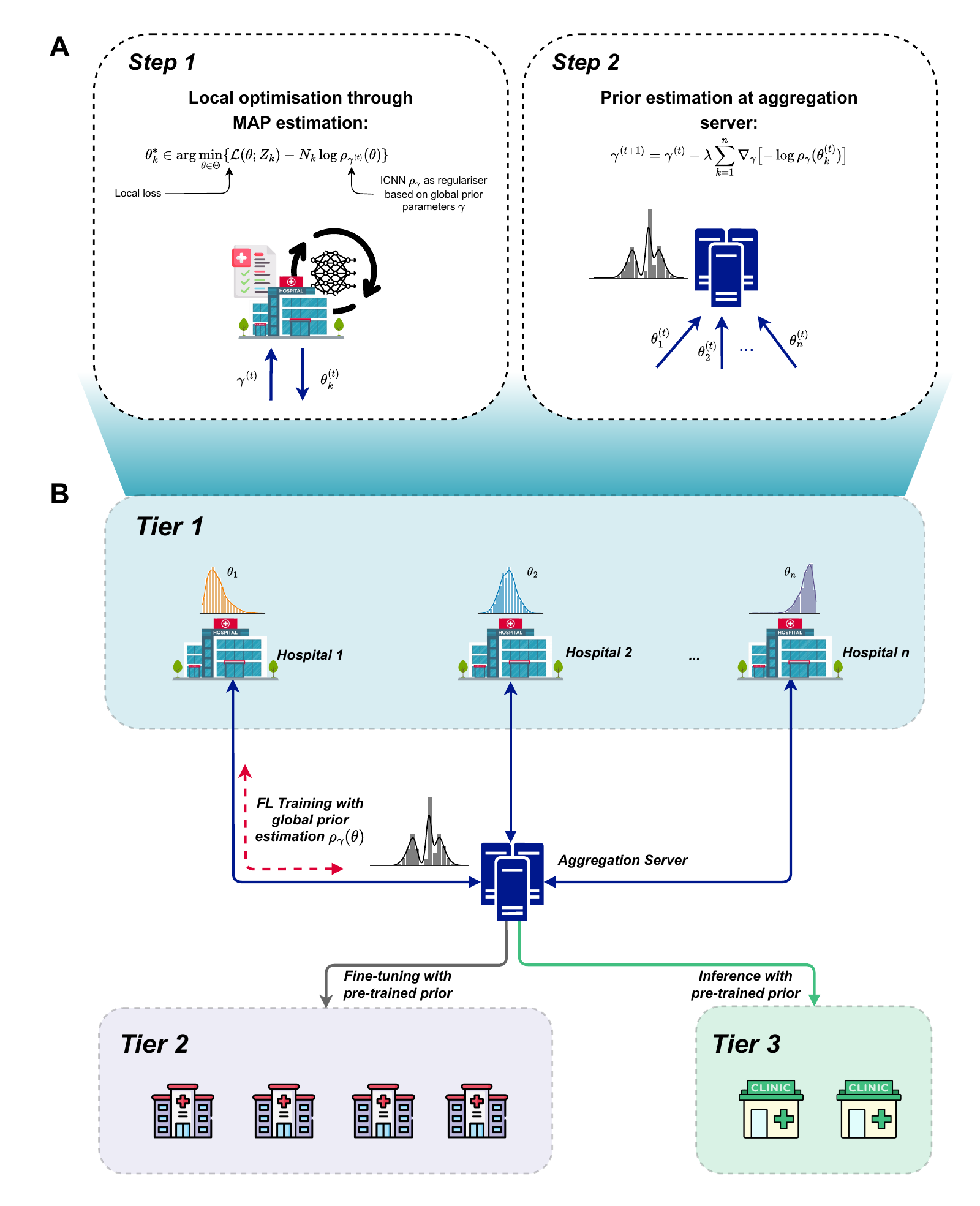}
\caption{
\textbf{FedMAP architecture diagram.}
\textbf{A,} Local training (step 1) and server prior estimation (step 2) of the proposed FedMAP algorithm for FL.
\textbf{B,} Diagram showcasing Tier 1 (FL), Tier 2 (local fine-tuning, for sites with limited IT networking resources), and Tier 3 (local inference-only use, for sites with limited IT networking resources and limited data availability) of FedMAP's multi-tier service design.
}
\label{fig:fedmap}
\end{figure}

\clearpage
\section*{Results}
We compared FedMAP against popular FL and PFL frameworks for both performance and equity. All FL methods were evaluated using three large-scale clinical datasets: CPRD~\cite{xu_prediction_2021} (258,688 patients from 387 practices), INTERVAL~\cite{angelantonio_efficiency_2017} (31,949 blood donors from four centres), and eICU~\cite{pollard_eicu_2018} (44,842 patients from 150 hospitals) (Table~\ref{tab:dataset_characteristics}). These datasets represent clinically relevant FL deployment scenarios spanning different institutional scales, from multi-centre trial configurations (INTERVAL: 4 sites) to enterprise-scale networks approaching realistic deployment conditions (CPRD: 387 sites). The clinical prediction tasks comprise 10-year cardiovascular disease (CVD) risk prediction (CPRD), iron deficiency detection (INTERVAL), and mortality prediction (eICU). These datasets exhibit substantial non-IID characteristics representative of real-world federated healthcare scenarios (see Materials \& Methods).\\

\noindent
For performance evaluation, we report concordance index (C-index) for CPRD, balanced accuracy for INTERVAL and AUROC for eICU, which follow established methodological frameworks from Xu et al~\cite{xu_prediction_2021}, Kreuter et al~\cite{kreuter_artificial_2025}, and Elhussein and Gürsoy~\cite{pmlr-v219-elhussein23a}, respectively. Statistical comparisons employed the Wilcoxon signed-rank test to assess site performance distributions across methods. Complete statistical analyses are presented in Tables~\ref{tab:comparison_cprd_fl}--\ref{tab:comparison_eICU_inf}. INTERVAL results were excluded from statistical testing due to insufficient sample sizes. We evaluated model performance across the three deployment tiers: Tier 1 (T1), Tier 2 (T2), and Tier 3 (T3). With three paired observations for T1 comparisons, the discrete nature of the Wilcoxon test statistic prevents achieving p $<0.05$ regardless of effect magnitude and T2 and T3 have only single sites precluding paired testing entirely.


\subsection*{Tier 1 performance: full FL}
Across all three datasets, FedMAP achieved the highest median performance amongst the evaluated methods  (Fig.~\ref{fig:performance}A). On the CPRD dataset (185 T1 sites), the standard FL baseline FedAvg achieved a median C-index of 0.667 (IQR: 0.060), whilst individualised PFL methods demonstrated variable improvements: pFedME reached 0.696 (IQR: 0.068), pFedBayes 0.692 (IQR: 0.066), pFedGP 0.691 (IQR: 0.058), and pFedFDA 0.502 (IQR: 0.137). Individual site training yielded a median of 0.690 (IQR: 0.072). FedMAP achieved a median C-index of 0.717 (IQR: 0.059), representing a \SI{3}{\percent} improvement over the next-best performing method, pFedME. FedMAP's performance differed significantly from pFedME ($p<0.001$, median of paired differences (MD) $=0.019$), FedAvg ($p<0.001$, MD $=0.050$), and individual training ($p<0.001$, MD $=0.024$).\\

\noindent
On the INTERVAL dataset (3 T1 sites), FedAvg achieved a median balanced accuracy of 0.760 (IQR: 0.013), with PFL methods showing mixed performance: pFedBayes reached 0.773 (IQR: 0.017), pFedGP 0.771 (IQR: 0.015), pFedFDA 0.768 (IQR: 0.010), and pFedME 0.737 (IQR: 0.025). Individual training achieved 0.779 (IQR: 0.025). FedMAP achieved the highest median balanced accuracy of 0.786 (IQR: 0.013). The eICU dataset (51 sites) showed similar patterns, with FedAvg achieving a median AUROC of 0.787 (IQR: 0.124), PFL methods ranging from pFedME's 0.772 (IQR: 0.107) to pFedGP's 0.732 (IQR: 0.173), and individual training at 0.701 (IQR: 0.139). FedMAP achieved a median AUROC of 0.808 (IQR: 0.116), significantly outperforming individual training ($p<0.001$, MD $=0.098$) and demonstrating improvement over FedAvg ($p=0.057$, MD $=0.019$).

\subsection*{Tier 2 performance: local fine-tuning}

FedMAP achieved the highest median performance across all three datasets in the fine-tuning scenario (Fig.~\ref{fig:performance}B).
On the CPRD dataset, FedMAP achieved a median  C-index of 0.708 (IQR: 0.089), compared to FedAvg's (0.663, IQR: 0.094; $p<0.001,\text{MD}=0.045$), representing a \SI{6.8}{\percent} improvement. Among PFL methods, pFedGP achieved 0.692 (IQR: 0.103), pFedBayes 0.681 (IQR: 0.091), whilst pFedME (0.663, IQR: 0.111) and pFedFDA (0.499, IQR: 0.131) showed lower performance. FedMAP significantly outperformed individual training (0.660, IQR: 0.098; $p<0.001,\text{MD}=0.056$).
For the single-site INTERVAL fine-tuning scenario (solely new blood donors), FedMAP achieved the highest balanced accuracy of 0.809, outperforming FedAvg (0.694). Among PFL methods, pFedME achieved 0.785, pFedGP 0.757, pFedBayes 0.739, and pFedFDA 0.708.
On the eICU dataset, FedMAP achieved the highest median AUROC of 0.778 (IQR: 0.133), outperforming FedAvg (0.704, IQR: 0.151; $p=0.001,\text{MD}=0.077$). Among PFL methods, pFedBayes achieved 0.753 (IQR: 0.198; $p=0.040,\text{MD}=0.038$), pFedGP and pFedME both 0.725, and pFedFDA 0.705. FedMAP also outperformed individual training (0.694, IQR: 0.171; $p=0.309,\text{MD}=0.049$).

\begin{figure}
    \centering
    \includegraphics[width=1\linewidth]{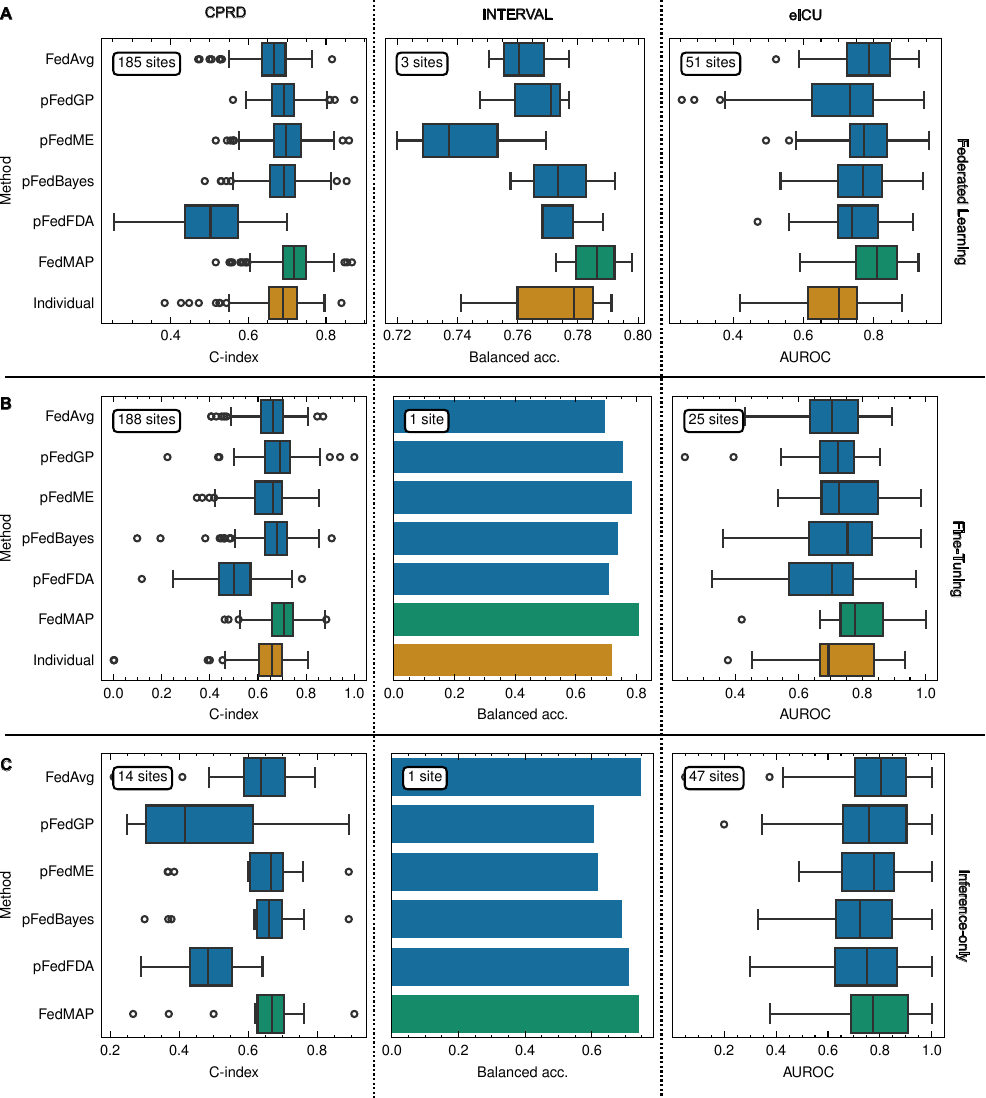}
    \caption{\textbf{Model performance across deployment tiers.}
    \textbf{a,} Tier 1 (T1): full federated training.
    \textbf{b,} Tier 2 (T2): local fine-tuning.
    \textbf{c,} Tier 3 (T3): inference-only.
    Performances are measured on local test sets. The number of sites in each group is indicated in the top left of the plots.
    Boxplot boxes indicate median and first-to-third quartile range of the data. Whiskers extend to 1.5 IQRs.
    IQR, interquartile range.
    }
    \label{fig:performance}
\end{figure}

\subsection*{Tier 3 performance: Inference-only deployment}
In the inference-only scenario, performance differences between methods were substantially reduced compared to T1 and T2 scenarios (Fig.~\ref{fig:performance}C). On CPRD, FedAvg achieved a median C-index of 0.636 (IQR: 0.118; $p = 0.863$ vs FedMAP, MD $= 0.048$). Among PFL methods, pFedME achieved 0.667 (IQR: 0.095; $p = 1$ vs FedMAP, MD $= 0.005$), pFedGP 0.416 (IQR: 0.226), and pFedBayes 0.660 (IQR: 0.166). FedMAP achieved a median C-index of 0.670 (IQR: 0.078). No statistically significant differences were observed between methods.\\

\noindent
On the INTERVAL dataset, FedAvg achieved the highest balanced accuracy of 0.748, followed by FedMAP at 0.743, pFedME at 0.618, and pFedGP at 0.607. On the eICU dataset, FedAvg achieved a median AUROC of 0.803 (IQR: 0.195), followed by pFedME at 0.778 (IQR: 0.201), FedMAP at 0.773 (IQR: 0.219; $p=1$ vs FedAvg, MD $=-0.008$), and pFedGP at 0.758 (IQR: 0.246). No statistically significant differences were observed between methods on available T3 sites.

\subsection*{Relationship between individual site performance and FL benefit}
We examined the relationship between individual training performance and change when using FL across healthcare sites. Linear regression analysis between individual training performance and relative performance gain using FedMAP revealed a strong negative correlation across all datasets: eICU (Spearman $r=-0.668$, $p<0.001$, $n=51$ sites) 
and CPRD (Spearman $r=-0.435$, $p<0.001$, $n=185$ sites). This indicates that sites with weaker individual performance gained substantially more from FL than high-performing sites.\\

\noindent
Relative performance changes ranged from \SI{-15.2}{\percent} to $+\SI{74.9}{\percent}$ across eICU sites. Sites achieving individual AUROC below 0.65 demonstrated median improvements of \SI{28.4}{\percent} (IQR: \SIrange{16.1}{47.9}{\percent}), whilst sites with individual AUROC above 0.65 showed median improvements of \SI{9.4}{\percent} (IQR: \SIrange{4.0}{16.4}{\percent}) (Fig.~\ref{fig:eicu_auroc_reg_fedmap}). Similar patterns were observed for the other datasets and FL methods (Extra Data Fig.~\ref{fig:supp_regression}).

\begin{figure}
    \centering
    \includegraphics[width=0.6\linewidth]{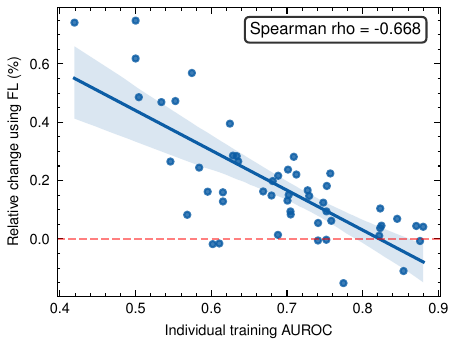}
    \caption{
    \textbf{FL performance gains versus individual site performance.}
    Each point represents an eICU site showing AUROC from individual training (x-axis) and relative change using FedMAP (y-axis). The solid line shows a linear regression fit with \SI{95}{\percent} confidence interval (shaded region). The dashed line indicates zero improvement.
    }
    \label{fig:eicu_auroc_reg_fedmap}
\end{figure}

\subsection*{Regional performance analysis}

\noindent
To assess whether FL can improve clinical decision support quality across regions with varying healthcare infrastructure, patient demographics, and baseline clinical outcomes,
we conducted an analysis of fine-tuning performance (T2) using the CPRD dataset partitioned by NHS England regions (Fig.~\ref{fig:england_map}).\\

\noindent
Individual training demonstrated substantial variation across regions.
London achieved the highest median C-index of 0.675, followed by South West at 0.666, whilst North East recorded the lowest at 0.601, representing a \SI{12.3}{\percent} performance range.\\

\noindent
FedMAP reduced regional performance variation to \SI{6.09}{\percent} (range: 0.682–0.723). North East of England improved from 0.601 to 0.687 median C-index (+\SI{14.3}{\percent}), Yorkshire and The Humber from 0.622 to 0.682 (+\SI{9.7}{\percent}), and North West from 0.662 to 0.692 (+\SI{4.5}{\percent}).
Alternative FL approaches exhibited different equity characteristics. pFedGP achieved performance ranging from 0.698 in London to lower values in other regions (\SI{8.6}{\percent} range). FedAvg demonstrated \SI{7.2}{\percent} regional variation. Both pFedME and pFedFDA showed either reduced absolute performance or larger regional disparities compared to FedMAP.



\begin{figure}[hbtp]
    \centering
    \makebox[\textwidth]{\includegraphics{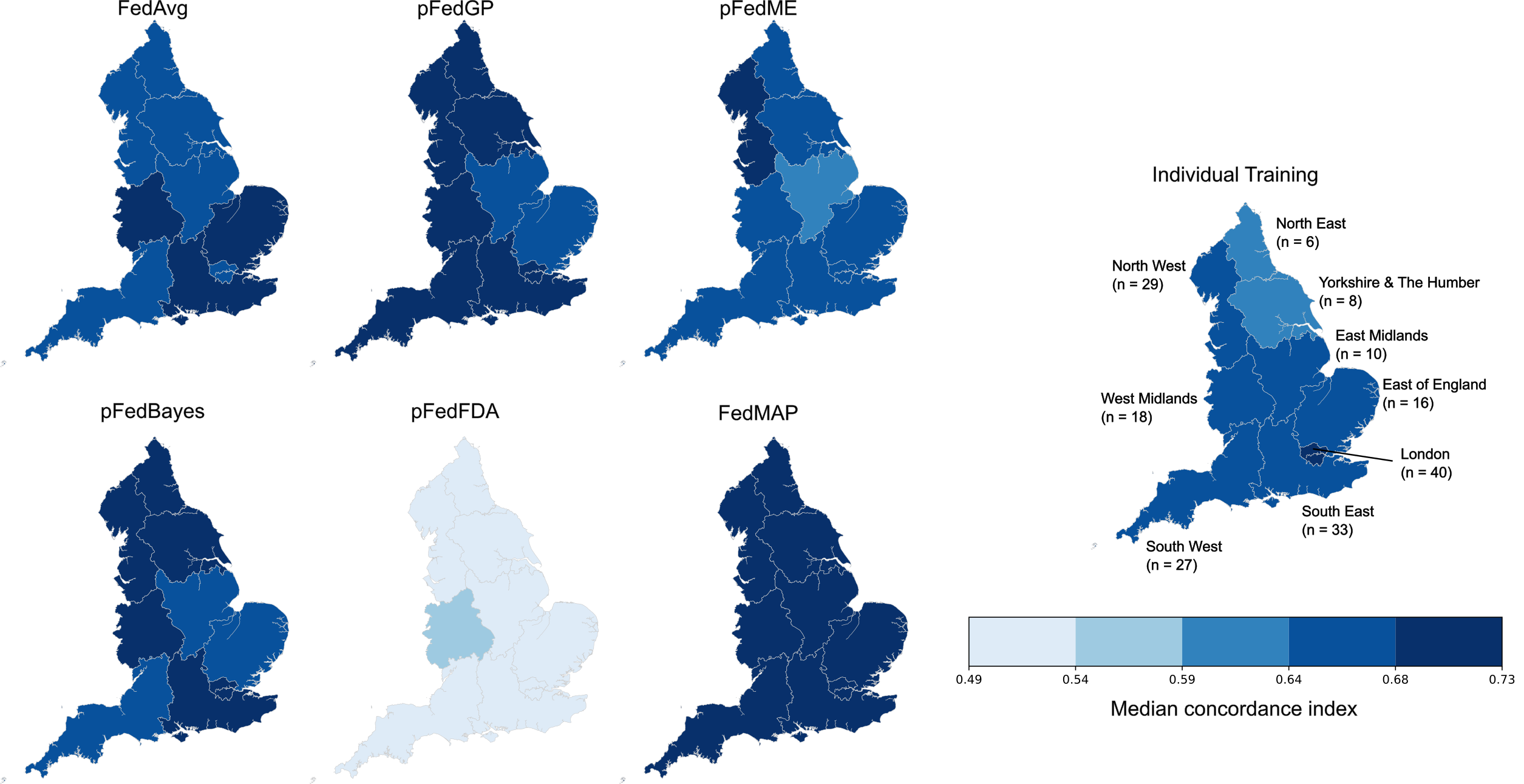}} 
    \caption{
    \textbf{10-year CVD prediction model performance collated by NHS England regions.}
    Colours indicate median C-index per region for the fine-tuning sites (Tier 2). 
    n, number of sites per region.
    }
    \label{fig:england_map}
\end{figure}


\clearpage
\section*{Discussion}\label{sec:discussion}
\noindent
We presented FedMAP, a personalised PFL framework that addresses the dual challenges of statistical heterogeneity and infrastructure disparities in healthcare systems. \\

\noindent
\textit{Theoretically}, we reframe PFL as a bi-level optimisation problem under a MAP framework with proven convergence guarantees. This moves beyond the simple parameter averaging of conventional methods. The core advance is the use of ICNNs to learn a data-driven global prior, which enables us to formally prove existence and convergence to a unique solution. This offers a theoretical foundation for personalisation that addresses combinations of data heterogeneity simultaneously in real world clinical data.\\

\noindent
\textit{Practically and clinically}, FedMAP achieved consistent performance gains across deployment scenarios. In full federated training (T1), FedMAP showed median improvements of 3\,\% (CPRD), 0.9\,\% (INTERVAL), and 2.7\,\% (eICU) over the next best method. These improvements helped reduce performance disparities: sites with weaker individual performance gained substantially more from federation, which narrowed the gap between high and low-performing sites. Linear regression revealed strong negative correlations between individual training performance and relative performance gain (eICU: $r = -0.668$, $p < 0.001$; CPRD: $r = -0.435$, $p < 0.001$), indicating FL preferentially improves outcomes where local data are most limited.\\

\noindent
Regional analysis demonstrated FedMAP's potential to address documented cardiovascular care inequalities. England shows significant regional disparities in cardiovascular disease outcomes, with a 51\,\% difference between the highest and lowest regional mortality rates, where northern regions experience significantly worse outcomes~\cite{cardiovascular_briefing_gajjar}. Individual hospital model training demonstrated performance patterns reflecting these real-world healthcare infrastructure inequalities. North East England and Yorkshire \& The Humber, with the highest cardiovascular mortality rates nationally,~\cite{gov_cardiovascular_2024} showed the largest performance gains under FedMAP (14.3\,\% and 9.7\,\%). This suggests FL might reduce outcome disparities by enabling resource-constrained regions to benefit from knowledge aggregated across better-resourced sites. FedMAP's fine-tuning approach enables resource-constrained hospitals to achieve prediction performance comparable to major specialist centres, potentially addressing documented regional disparities in cardiovascular care quality~\cite{stoye_distribution_2022}. This capability could support NHS England's cardiovascular disease strategy~\cite{UKGovernment2025} , while accommodating the financial constraints facing regional Integrated Care Systems~\cite{bliss2024state}.\\


\noindent
\textit{Operationally}, FedMAP presents a realistic view of healthcare infrastructure setup through its three tier design. Most of the research on FL neglects the operational fact that healthcare organisations have significantly varied computational resources~\cite{rieke_future_2020, kaissis_secure_2020}. Clinics serving underrepresented populations are often the least equipped to participate in advanced AI development~\cite{nong_urgency_2025}. By providing distinct pathways from full federated training for academic medical centres (T1) to using the pre-trained global model for resource constrained sites (T3), FedMAP offers a scalable and inclusive ecosystem. This service oriented model is a critical step towards the equitable clinical translation of FL. This tiered deployment model creates a system where well-resourced sites contribute training data and computational resources and smaller sites benefit from improved predictive models without requiring local ML infrastructure or expertise, which addresses the inequity where hospitals serving vulnerable populations have the least access to AI-driven clinical decision support~\cite{nong_urgency_2025}. Looking beyond high-income healthcare systems, our approach is particularly relevant for low- and middle-income countries which face critical shortages of specialist physicians and diagnostic infrastructure~\cite{agyeman-manu_prioritising_2023}. FL allows these regions access to the global medical expertise as well as keeping their data private and secure.\\

\noindent
Looking forward, this study has limitations which define a clear path for future research. Firstly, the bi-level optimisation is more computationally intensive than FedAvg. Site level computational requirements remain minimal as each site performs standard gradient-based optimisation with negligible overhead from the ICNN regulariser. However, we demonstrated scalability to hundreds of sites (387 in CPRD, 150 in eICU) and scaling to thousands of simultaneously active sites requires further algorithmic optimisation. Secondly, for T3 deployment, detecting distribution shift between new inference-only sites and the federated training cohort remains an open challenge. Future work should develop methods to quantify when a site's data distribution falls outside the training manifold, potentially triggering recommendations for local fine-tuning to ensure reliable predictions. Thirdly, the learned ICNN prior currently functions as a ``black box''. A significant future direction is to develop methods to interpret this prior, as it can contain valuable latent information about the relationships and clusters within the healthcare network itself. Lastly, although FedMAP does not share raw data, it lacks formal privacy guarantees. The transmission of high dimensional model parameters can still be vulnerable to model inversion or membership inference attacks. Future work should integrate techniques such as differential privacy, secure aggregation and fine-grained access control to provide quantifiable security. \\

\noindent
FedMAP addresses both technical and operational barriers to healthcare FL deployment. By achieving consistent improvements as well as accommodating diverse institutional capabilities, which provides a path from proof of concept to production systems. The preferential improvements in underperforming regions suggest FL can reduce rather than amplify healthcare disparities.

\section*{Methods}\label{sec6}
\subsection*{FedMAP Framework: Theoretical Foundations}
Let us consider a prediction task in which the input space is denoted by $\mathcal{X}$ and the output space is $\mathcal{Y}$. In the FL framework, we have a collection of datasets $Z_k$ available for development, where
$$
Z_k = \{ (x^{(i)}_k, y^{(i)}_k)\}_{i=1}^{N_k} \in (\mathcal{X}\times \mathcal{Y})^{N_k}, \qquad
\text{for} \quad k = 1,\ldots , q,
$$
each of them consisting of an IID sampling from a probability distribution $\mathcal{D}_k$ in $\mathcal{X}\times \mathcal{Y}$.

\noindent
The learning procedure for FL with a central aggregator (most methods) consists of two stages, which can be repeated iteratively:
\begin{enumerate}
    \item \textit{Local training:} Each site $k$ trains a model by optimising model parameters $\theta_k$ to minimise a loss function $\mathcal{L}(\theta_k; \mathcal{Z}_k)$ that quantifies prediction error on its local dataset $\mathcal{Z}_k$, optionally incorporating knowledge from the global model.
    \item \textit{Aggregation:} The central server updates a global model based on the trained model parameters from local sites.
\end{enumerate}
Each iteration of this two-stage process is known as a \textit{communication round}.\\

\noindent
Let us consider a parametric family of models with a finite number of parameters, denoted by
$$
\mathcal{H} := \bigl\{ 
\phi_\theta : \mathcal{X} \to \mathcal{Y} \;\big|\; \theta\in \Theta
\bigr\},
$$
where $\Theta \subset \mathbb{R}^d$ is the parameter space.
The local training procedure consists of minimising, over $\theta\in \Theta$, a loss functional $\mathcal{L} (\theta, (x,y))$ which, to every model $\phi_\theta: \mathcal{X} \to \mathcal{Y}$ and any data point $(x,y)\in \mathcal{X}\times \mathcal{Y}$, assigns a nonnegative scalar that quantifies the discrepancy between the model prediction and the data. Since, typically, the output $y$ is assumed to be a random variable (e.g. a noisy version of ground truth $y^\ast$), the standard choice for
the loss functional during the local training is the negative log-likelihood, defined as
$$
\mathcal{L} (\theta, (x,y)) := -\log \mathbb{P} (y \mid \phi_\theta (x)),
$$
where $\mathbb{P} (y \mid \phi_\theta (x))$ denotes the likelihood of the random variable $y$ conditionally on the ground truth being $y^\ast = \phi_\theta (x)$. Considering empirical risk minimisation as the learning algorithm, the parameter $\theta_k^\ast$ for each local model $\phi_{\theta_k^\ast}$ can be written as the solution to the minimisation problem
\begin{equation}
\label{MLE problem}
\theta_k^\ast \;\in\; \arg \min_{\theta\in \Theta} \dfrac{1}{N_k} \sum_{i=1}^{N_k}
 \bigl[-\log \mathbb{P}\bigl( y_k^{(i)} \mid \phi_\theta (x_k^{(i)}) \bigr)\bigr] \, ,
\end{equation}
corresponding to a maximum likelihood estimator (MLE), i.e., the most likely model $\phi_\theta$ in $\mathcal{H}$ given the dataset $Z_k$.\\

\noindent
In non-PFL frameworks, a global model is generated during aggregation by combining all the models trained on local data. Then, the global model is used only as an initial guess for the next communication round. However, the minimisation problem in \eqref{MLE problem} is independent of the initialisation, and therefore, the information leveraged during the aggregation risks being forgotten during the local training.
To prevent this issue, in most FL methods one does not aim at solving \eqref{MLE problem} to completion during the local training. Instead, one would typically warm start from a global initialisation but limit the divergence of local parameters through regularisation.
In our approach, we remedy the issue of forgetting the global model by considering a parametric regularisation term in \eqref{MLE problem}, learned during the aggregation stage by combining the different local models. In this way, local updates remain anchored to global knowledge.

\subsubsection*{Local training as MAP estimation}
The MLE approach described above can be enhanced by incorporating a parametric regulariser into the minimisation problem in \eqref{MLE problem}, with the regulariser being estimated during the aggregation step.
This new approach can be interpreted in terms of MAP estimation.
We consider a parametric family of regulariser functions
$$
\mathcal{G} := \bigl\{ \mathcal{R} (\cdot; \gamma) : \Theta \to \mathbb{R}^+ \;\big|\; \gamma\in \Gamma \bigr\},
$$
where $\Gamma \subset \mathbb{R}^p$ is the parameter space for the regulariser.\\

\noindent
Let us denote by $\gamma^{(t)}$ the regulariser parameters after $t$ communication rounds.
If we apply regularised empirical risk minimisation over the dataset $Z_k = \{(x_k^{(i)},\,y_k^{(i)})\}_{i=1}^{N_k}$, the $k$-th site updates its local model to
\begin{equation}
\label{MAP estimation}
\theta_k^{(t+1)} \;\in\; 
\arg \min_{\theta\in \Theta} \Biggl\{
 \dfrac{1}{N_k} \underbrace{ \sum_{i=1}^{N_k} 
-\log \Bigl[\mathbb{P}\bigl(y_k^{(i)} \mid \phi_\theta (x_k^{(i)})\bigr)\Bigr]}_{\displaystyle \mathcal{L}(\theta;Z_k)}
\;+\; 
\mathcal{R} (\theta; \gamma^{(t)}) 
\Biggr\}.
\end{equation}
The parameter $\gamma$ in the regulariser $\mathcal{R} (\cdot; \gamma)$ is learned from all sites' datasets. This is done indirectly, during the aggregation stage, by combining all the local models $\{\theta_k^{(t)}\}_{k=1}^q$ from the previous communication round.\\

\noindent
The minimisation problem \eqref{MAP estimation} can be viewed as a MAP estimation associated to a prior distribution in the parameter space $\Theta$. Let us define the function
$$
\rho_\gamma (\theta) := e^{-\mathcal{R} (\theta; \gamma)}. 
$$
Under suitable growth assumptions\footnote{It is sufficient to assume that $\theta\mapsto \mathcal{R} (\theta; \gamma)$ grows at least linearly, i.e. there exist constants $c_1,c_2>0$ such that $\mathcal{R}(\theta; \gamma) \geq c_1 \|\theta\| - c_2$ for all $\theta \in \Theta$. This implies that the integral $\int_{\Theta} e^{-\mathcal{R}(\theta; \gamma)} d\theta$ is finite. Alternatively, one can also assume that the parameter space $\Theta$ is compact, in which case, no growth assumption on $\mathcal{R}(\cdot; \gamma)$ is required.} on $\mathcal{R} (\cdot; \gamma)$, the function
$$
\bar{\rho}_\gamma (\theta) :=  \dfrac{\rho_\gamma (\theta)}{\int_\Theta \rho_\gamma (\xi) d\xi}
$$
defines a probability distribution over $\Theta$.\\

\noindent
Given the prior probability distribution $\bar{\rho}_\gamma$, for any $(x,y)\in \mathcal{X}\times \mathcal{Y}$ and $\theta\in \Theta$, we can use Bayes’ theorem to compute (up to the normalising constant $\int_\Theta \rho_\gamma (\xi) d\xi$) the posterior probability 
$$
\mathbb{P} (\theta \mid (x,y)) 
\;\propto\; 
\mathbb{P}\bigl((x,y)\mid \theta\bigr)\,\rho_\gamma(\theta)
\;=\; 
\mathbb{P}\bigl(y \mid \phi_\theta(x)\bigr)\,\rho_\gamma(\theta) \, ,
$$
i.e. the probability of $\theta$ given $(x,y)$ and the prior $\rho_\gamma$.
Hence, given the dataset $Z_k = \{ (x_k^{(i)}, y_k^{(i)}) \}_{i=1}^{N_k}$, the negative log-likelihood of $\theta$ can be written (up to an additive constant) as 
\begin{align*}
-\log \left( \mathbb{P} (\theta | Z_k) \right)
& = 
-\sum_{i=1}^{N_k} \log\Bigl[\mathbb{P}\bigl(y_k^{(i)} \mid \phi_\theta (x_k^{(i)})\bigr)\Bigr]
\;-\; N_k\log\,\rho_\gamma(\theta)\\
&= \mathcal{L} (\theta; Z_k) \;+\; N_k \mathcal{R} (\theta;\gamma).
\end{align*}
Minimising the above quantity (scaled by $1/N_k$) yields \eqref{MAP estimation}.
Therefore, $\theta_k^{(t+1)}$ is the parameter that maximises the posterior probability given the dataset $Z_k$ and the prior distribution $\rho_{\gamma^{(t)}}$.


\subsubsection*{Estimating the prior during aggregation}

After each site trains its local model $\theta_k^{(t)}$ via \eqref{MAP estimation}, the server collects $\{\theta_k^{(t)}\}_{k=1}^q$. The global parameter $\gamma \in \Gamma$ is then updated according to~\eqref{prior MLE}, which minimises the negative log of the joint prior‐likelihood over these local parameters. Taking into account that each $\theta_k^{(t)}$ is estimated by using $N_k$ data points, we can write 
\[
-\log \left[ \mathbb{P} (\theta_1^\ast, \ldots, \theta_q^\ast | \gamma)\right] = 
\sum_{k=1}^q \bigl[-N_k\log \rho_\gamma(\theta_k^\ast)\bigr] = \sum_{k=1}^q N_k \mathcal{R} \bigl(\theta_k^{(t)}; \gamma \bigr).
\]
In practice, at communication round $t$, given the current $\gamma^{(t)}$ and the local models $\bigl\{\theta_k^{(t)}\bigr\}_{k=1}^q$, we take a gradient step as follows:
\begin{equation}
\label{prior MLE}
\gamma^{(t+1)} 
\;=\; 
\gamma^{(t)} 
\;-\; 
\lambda \sum_{k=1}^q N_k \nabla_{\gamma} \mathcal{R} (\theta_k^{(t)}; \gamma^{(t)}) \, , 
\end{equation}
where $\lambda>0$ is a suitable step size. Because $\mathcal{R} (\theta; \gamma)$ is a differentiable function on $\Theta\times \Gamma$, as stated in Assumption~\ref{assum: convexity}, the gradient $\nabla_{\gamma}\mathcal{R}(\theta; \gamma)$ is obtained by backpropagating through the chosen parametric form of the prior. Note that no raw data leaves a site, only its parameter vector $\theta_k^{(t)}$ and associated weights are transmitted to the server.

\subsubsection*{FL as bi-level optimisation}\label{subsec3}

Here, we provide convergence guarantees for FedMAP under suitable convexity assumptions specified in Assumption~\ref{assum: convexity}.
In particular, we prove that the updates of the local models $\{\theta_k^{(t)}\}_{k=1}^q$ and the regulariser parameter $\gamma^{(t)}$, obtained by iterating \eqref{MAP estimation} and \eqref{prior MLE}, converge to the solution to the bi-level optimisation problem:
\begin{equation}
\label{bi-level optimisation}
   \begin{aligned}
   \underset{\theta_k \in \Theta}{\operatorname{minimise}} & \ \mathcal{L} ({\theta_k}; Z_k) + N_k \mathcal{R} (\theta_k; \gamma^\ast) \quad \forall k=1, \ldots, q \\
   \text{s.t.} \quad & \gamma^\ast \in \arg\min_{\gamma\in \Gamma} \left( \sum_{k=1}^q N_k \mathcal{R} (\theta_k; \gamma) \right).
   \end{aligned}
\end{equation}
\noindent
Therefore, local training \eqref{MAP estimation} and aggregation \eqref{prior MLE} can be seen as an alternating strategy to approximate the solution to this bi-level optimisation problem. Local training addresses the upper-level problems (optimising local models given the current regulariser), while the aggregation step addresses the lower-level problem (learning the optimal regulariser parameters given the local models). 

The main idea is to first prove that the updates of the regulariser parameter $\gamma(t)$, given in \eqref{MAP estimation} and \eqref{prior MLE}, correspond to gradient descent iterations applied to the function $M(\gamma)$ defined as
\begin{equation}
\label{M of gamma def}
M(\gamma) := \sum_{k=1}^q M_k (\gamma ; Z_k), \quad \text{where} \quad
M_k(\gamma; Z_k) = \min_{\theta\in \Theta} \left\{\mathcal{L} (\theta; Z_k) + N_k \mathcal{R} (\theta; \gamma)\right\}.
\end{equation}
Then, using the convexity assumptions on $\mathcal{L}(\theta; Z_k)$ and $\mathcal{R}(\theta; \gamma)$ one can prove that the function $M(\gamma)$ is strongly convex, and hence, with a suitable learning rate, gradient descent iterations converge to the unique minimiser of $M(\gamma)$, denoted by $\gamma^\ast$. Finally, we conclude that the local models $\theta_k^\ast$ associated to the local training \eqref{MAP estimation} with regulariser parameter $\gamma^\ast$ provide the solution to the bi-level optimisation problem \eqref{bi-level optimisation}.

Next we present the statement of our main theoretical contribution. The complete proof is given in Appendix \ref{secA1}.

\begin{theorem}
    \label{thm:main paper}
    Let $\Theta\subset \mathbb{R}^d$ be compact and convex, and let $\Gamma\subset \mathbb{R}^p$ be convex. Assume that $\theta\mapsto \mathcal{L}(\theta; Z_k)$ is continuous and convex, and that $(\theta, \gamma)\mapsto \mathcal{R}(\theta; \gamma)$ is differentiable and strongly convex\footnote{See \eqref{strong convexity} in Appendix \ref{secA1} for the precise definition of strong convexity.}. Then,
    \begin{enumerate}
        \item the FedMAP updates of $\gamma^{(t)}$, given by \eqref{MAP estimation} and \eqref{prior MLE}, correspond to gradient descent iterations applied to the function $M(\gamma)$ defined in \eqref{M of gamma def}, i.e.
        $$
        \gamma^{(t+1)} = \gamma^{(t)} - \lambda \nabla_\gamma M(\gamma),
        $$
        \item the function $M(\gamma)$ is strongly convex in $\Gamma$.
        \item the unique minimiser $\gamma^\ast$ of $M(\gamma)$ is such that the local models $(\theta_1^\ast, \ldots , \theta_q^\ast)$ given by
$$\theta_k^\ast := \theta_k^{\ast}(\gamma^\ast) = \arg\min_{\theta\in \Theta} \left\{ \dfrac{1}{N_k}\mathcal{L} (\theta; Z_k) + \mathcal{R} (\theta; \gamma^\ast)\right\} , \qquad \text{for all}\ k=1,\ldots , q,$$
are the unique solution of the bi-level optimisation problem \eqref{bi-level optimisation}.
\end{enumerate}
\end{theorem}

\noindent
The fact that the updates of $\gamma^{(t)}$ correspond to gradient descent iterations of the strongly convex function $M(\gamma)$ imply that, by choosing a sufficiently small learning rate $\lambda$, $\gamma^{(t)}$ converges, as $t\to \infty$, to the unique minimiser of $M(\gamma)$. Hence, in view of the third point in Theorem \ref{thm:main paper}, we deduce the following corollary.

\begin{corollary}
    \label{cor: convergence to bi-level opt}
    Under the assumptions of Theorem \ref{thm:main paper}, there exists $\lambda_0>0$ such that, if we take $\lambda\in (0,\lambda_0)$ in \eqref{prior MLE}, then the FedMAP iterates for $\theta_k^{(t)}$ and $\gamma^{(t)}$ given by \eqref{MAP estimation} and \eqref{prior MLE} converge, as $t\to \infty$, to the unique solution to the bi-level optimisation problem \eqref{bi-level optimisation}.
\end{corollary}

\noindent
The function $M_k (\gamma; Z_k)$ in \eqref{M of gamma def} can be seen as an envelope function~\cite{danskin_theory_1966, bernhard_theorem_1995} associated with the loss $\mathcal{L} (\theta;Z_k)$ and the parametric regulariser $\mathcal{R} (\theta; \gamma)$. Note that, with the particular choice of the quadratic regulariser $\mathcal{R} (\theta; \gamma) = \| \theta - \gamma\|^2$, the function $M_k (\gamma; Z_k)$ would be precisely to the so-called Moreau envelope of $\mathcal{L} (\theta;Z_k)$. Here, we consider a more general class of regularisers, allowing more adaptability for the prior distribution in the parameter space.
Minimising a linear combination of envelope functions such as $M(\gamma)$ to train a global model differs fundamentally from standard FL approaches methods such as FedAvg~\cite{mcmahan_communication-efficient_2016}  and FedProx~\cite{li2020federated}, which minimise a linear combination of local loss functions, i.e., $F (\theta) = \sum_{k=1}^q w_k \mathcal{L} (\theta; Z_k)$.
In our implementation, the convexity assumptions on the parametric  regulariser are guaranteed by using ICNNs~\cite{pmlr-v70-amos17b}.

\subsection*{Algorithmic Implementation}
\label{sec:fedmap_algo}

We parametrise the local models using neural networks. Specifically, we define the hypothesis set as a family of neural networks:
$$
\phi (\cdot \, ; \theta): \mathcal{X} \to \mathcal{Y}, \qquad \theta \in \Theta := \mathbb{R}^d.
$$
As a parametric regulariser $\mathcal{R} (\theta; \gamma)$, we use a combination of an \textbf{Input Convex Neural Network (ICNN)}~\cite{pmlr-v70-amos17b} with parameters $\psi \in \Psi$ and a quadratic penalty term with respect to a global model that we call $\mu$. More precisely, we consider parametric regularisers of the form 
\begin{equation}
\label{ICNN def}
\mathcal{R}(\theta; \mu, \psi) = f_\psi(\theta, \mu) + \alpha\|\theta - \mu\|^2 + \epsilon(\|\theta\|^2 + \|\mu\|^2),
\end{equation}
where the parameter is of the form $\gamma = (\mu, \psi) \in \Gamma := \Theta\times \Psi$.
Here, $f_\psi: \Theta\times \Theta \to \mathbb{R}$ is an NN with non-negative weights (except for the first layer) and convex activation functions. The hyperparameters $\alpha,\varepsilon > 0$ provide strong convexity.
The ICNN is designed to be convex with respect to both $\theta$ and $\mu$ for any fixed $\psi$.
The use of an ICNN for the parametric regulariser enhances the convexity assumption in Theorem \ref{thm:main paper}, and allows FedMAP to learn an adaptive regularisation landscape that can capture the complex inter-site relationships in heterogeneous federated systems. This adaptive nature is crucial for effectively addressing combinations of non-IID data distributions.\\

\noindent
Unlike traditional quadratic regularisers, that impose a fixed geometric structure~\cite{li2020federated}, the ICNN component $f_\psi$ learns a data-driven regularisation landscape that can capture complex inter-site relationships and adapt to the specific non-IID characteristics of the federated dataset.\\

\noindent
While some PFL approaches employ fixed-form regularisers such as the L2 penalty corresponding to a Gaussian prior with fixed variance, these pre-specified structures often lack the adaptability to capture the diverse inter-site relationships in heterogeneous federated systems. FedMAP defines its regulariser $\mathcal{R}(\theta; \mu, \psi)$ using an ICNN~\cite{pmlr-v70-amos17b}, a class of neural networks specifically constructed such that their output function is convex with respect to a designated subset of their inputs. In our framework, $\mu \in  \Theta$ represents the global model parameters, serving as a reference point around which local parameters $\theta$ are regularised. $\psi \in \Psi$ are the learnable parameters of the ICNN itself. The ICNN's design ensures that for any fixed $\psi$, the function $(\theta, \mu) \mapsto \mathcal{R}(\theta; \mu, \psi)$ is convex. The prior distribution associated to the model parameter $\theta$ is then defined as 
$$
\rho_{\gamma} (\theta) = \exp \left(- \mathcal{R}(\theta; \mu, \psi)\right), \qquad
\text{where $\gamma = (\mu, \psi)$.}
$$

\noindent
A significant advantage of employing an ICNN is its capacity to learn a more adaptive and data-driven regularisation landscape compared to a fixed penalty. Both the global model parameters $\mu$ and the ICNN parameters $\psi$ are updated at the server during the aggregation step. This allows FedMAP to learn not only a global consensus model but also the appropriate structure of regularisation tailored to the specific FL task.\\

\noindent
\subsection*{Three-Tier Service Architecture for Diverse Institutional Capabilities}
\subsubsection*{T1: Full FL}
The complete FedMAP algorithm consists of three main steps: (i) initialisation, (ii) local optimisation, and (iii) global aggregation, as outlined in Algorithm~\ref{algo:map_overview}.

\paragraph{Initialisation} Site $j$ is randomly selected, and its parameters are used to initialise the global parameters $\mu^{(0)}$ and the local parameters for all sites, i.e., $\theta_k^{(0)} = \theta_j^{(0)}$ for all $k$. The ICNN parameters $\psi^{(0)}$ are also initialised (e.g., using a standard neural network initialisation scheme). These initial parameters $\gamma^{(0)} = (\mu^{(0)}, \psi^{(0)})$, and $\{\theta_k^{(0)}\}_{k=1}^q$ are then shared to all partner sites, ensuring a common starting point for all sites.

\begin{algorithm}[htb]
\caption{FedMAP - T1: Full FL}
\label{algo:map_overview}
\begin{algorithmic}[1]
\State \textbf{Input:} $q$ (Total number of sites), $\Psi$ (ICNN parameter space), $T$ (Number of communication rounds)
\State \textbf{Initialisation:}
\State Randomly select site $j$ from $\{1, \ldots, q\}$
\State Initialise $\mu^{(0)}$ and all $\theta_k^{(0)}$ based on site $j$'s model.  Randomly initialise $\psi^{(0)} \in \Psi$.
\State Broadcast $\gamma^{(0)} = (\mu^{(0)}$, $\psi^{(0)})$, and $\{\theta_k^{(0)}\}_{k=1}^q$ to all sites.
\For{$t = 0$ to $T-1$}
    \For{$k = 1$ to $q$ \textbf{in parallel}}
        \State $\theta_k^{(t+1)}, \omega_k^{(t)} \leftarrow \Call{LocalOptimisation}{\theta_k^{(t)}, \mu^{(t)}, \psi^{(t)}, Z_k}$ \Comment{Algorithm~\ref{algo:local_optimisation}, with early stopping}
    \EndFor
        \State $\gamma^{(t+1)} = (\mu^{(t+1)}, \psi^{(t+1)}) \leftarrow \Call{GlobalAggregation}{\{\theta_k^{(t+1)}, \omega_k^{(t)} \}_{k=1}^q, \mu^{(t)}, \psi^{(t)}}$ \Comment{Algorithm~\ref{algo:global_aggregation}}
\EndFor
\State \textbf{Output:} Final parameters: global $\mu^{(T)}$, ICNN $\psi^{(T)}$, personalised $\{\theta_k^{(T)}\}_{k=1}^q$
\end{algorithmic}
\end{algorithm}

\paragraph{Local Optimisation} Each site $k$ optimises their model parameters $\theta_k^{(t+1)}$ by 

\begin{equation}
\begin{split}
    \label{eq:local_opt_objective_icnn}
    \theta_k^{(t+1)} & = \arg \min_{\theta \in \Theta} \left\{ \dfrac{1}{N_k} \sum_{i=1}^{N_k} -\log \left[\mathbb{P}\big(y_k^{(i)}|\phi_\theta(x_k^{(i)})\big) \right] + \mathcal{R}(\theta; \mu^{(t)}, \psi^{(t)}) \right\} \\
    & = \arg \min_{\theta \in \Theta} \left\{ \frac{1}{N_k}  \mathcal{L} \left( \theta; Z_k\right) + \underbrace{\mathcal{R}(\theta; \mu^{(t)}, \psi^{(t)})}_{\text{Learned regulariser}} \right\} \; ,
\end{split}
\end{equation}

\noindent
minimising the negative log-likelihood of the posterior distribution. This objective combines the local data likelihood with the global regulariser
where $\mathcal{L} \left( \theta; Z_k\right)$ is the local loss function, and $N_k$ is the number of data points in $Z_k$. The term $\mathcal{R}(\theta; \mu^{(t)}, \psi^{(t)})$, defined by the ICNN, penalises deviations of the local model $\theta$ from the global model $\mu^{(t)}$, where the ICNN parameters $\psi^{(t)}$ adaptively learn the regularisation landscape that captures inter-site relationships in the federated network. Each site iteratively updates its local parameters $\theta_k$ for a fixed number of local epochs $e$, as detailed in Algorithm~\ref{algo:local_optimisation}.\\

\noindent
After optimising the local model, each site computes a weighting factor
\begin{equation}
    \label{eq:weighting_factor_icnn}
    \omega_k^{(t)} = \mathbb{P}(Z_k|\theta_k^{(t+1)}) \times \exp(-\mathcal{R}(\theta_k^{(t+1)}; \mu^{(t)}, \psi^{(t)})) \; .
\end{equation}

\noindent
This factor reflects the quality of the local model $\theta_k^{(t+1)}$ under both the local data likelihood and the current global prior (defined by $\mu^{(t)}$ and $\psi^{(t)}$),
such that $\omega_k^{(t)} \propto \mathbb{P}(\theta_k^{(t+1)} \mid Z_k)$, which favours sites whose local models explain their data well and remain close to the globally learned prior encoded by $\mathcal{R}$.
The locally optimised parameters $\theta_k^{(t+1)}$ and their corresponding weighting factor $\omega_k^{(t)}$ are then transmitted to the central server.

\begin{algorithm}[htb]
\caption{Local optimisation at site $k$}
\label{algo:local_optimisation}
\begin{algorithmic}[1]
\State \textbf{Input:} Local parameters $\theta_k^{(t)}$, global model parameters $\mu^{(t)}$, ICNN parameters $\psi^{(t)}$, local dataset $Z_k$, number of local epochs $e$.
\State Initialise $\theta \leftarrow \theta_k^{(t)}$.
\For{$epoch = 1$ to $e$}
    \State $\theta \leftarrow \arg \displaystyle\min_{\theta' \in \Theta} \left\{ \frac{1}{N_k} \mathcal{L} (\theta; Z_k) + \mathcal{R}(\theta'; \mu^{(t)}, \psi^{(t)}) \right\}$ \Comment{e.g., via SGD steps}
\EndFor
\State $\theta_k^{(t+1)} \leftarrow \theta$.
\State $\omega_k^{(t)} \leftarrow \mathbb{P}(Z_k|\theta_k^{(t+1)}) \times \exp(-\mathcal{R}(\theta_k^{(t+1)}; \mu^{(t)}, \psi^{(t)}))$.
\State \Return $\theta_k^{(t+1)}, \omega_k^{(t)}$.
\end{algorithmic}
\end{algorithm}

\paragraph{Global Aggregation} The server aggregates the optimised local model parameters $\{\theta_k^{(t+1)}\}_{k=1}^q$ from all sites to update the regulariser parameters $\gamma^{(t+1)} = (\mu^{(t+1)}, \psi^{(t+1)})$: the global model parameters $\mu^{(t+1)}$ and the ICNN parameters $\psi^{(t+1)}$, as detailed in Algorithm~\ref{algo:global_aggregation}.
The global model parameters $\mu^{(t+1)}$ are updated as a weighted average of the received local model parameters:
\begin{equation}
\label{eq:gamma_update_icnn_agg}
\mu^{(t+1)} = \frac{\sum_{k=1}^q \omega_k^{(t)} \theta_k^{(t+1)}}{\sum_{j=1}^q \omega_j^{(t)}}.
\end{equation}
The ICNN parameters $\psi^{(t+1)}$ are updated by performing multiple gradient descent steps on the lower-level objective of the bi-level optimisation. Specifically, $\psi$ is updated to minimise the weighted sum of regulariser values applied to the new local models $\theta_k^{(t+1)}$ with respect to the new global model $\mu^{(t+1)}$:
\begin{equation}
\label{eq:psi_update_icnn_agg}
\psi^{(t+1)} = \psi^{(t)} - \lambda_{\psi} \nabla_{\psi} \left[ \sum_{k=1}^q \frac{\omega_k^{(t)}}{\sum_{j=1}^q \omega_j^{(t)}} \mathcal{R}(\theta_k^{(t+1)}; \mu^{(t+1)}, \psi^{(t)}) \right],
\end{equation}
where $\lambda_{\psi}$ is a learning rate for the ICNN parameters.\\

\noindent
\textbf{Remark.} These updates in Equations~\eqref{eq:gamma_update_icnn_agg} and~\eqref{eq:psi_update_icnn_agg} are designed to approximate gradient descent steps on $M(\mu, \psi)$ as proven in Theorem~\ref{thm:main paper}. The parameter $\psi$ is updated by taking a weighted average of the gradients of each term $\mathcal{R}(\theta_k; \mu, \psi)$ with respect to $\psi$, whilst the global model parameter $\mu$ is updated by minimising the function $\mu \mapsto \sum_{k=1}^q \omega_k \| \theta_k - \mu\|^2$.

\noindent The updated global model parameters $\mu^{(t+1)}$ and ICNN parameters $\psi^{(t+1)}$ are then broadcast to all sites for the subsequent round of local optimisation.

\begin{algorithm}[htb]
\caption{Global aggregation at server}
\label{algo:global_aggregation}
\begin{algorithmic}[1]
\State \textbf{Input:} Local models $\{\theta_k^{(t+1)}\}_{k=1}^q$, weights $\{\omega_k^{(t)}\}_{k=1}^q$, current global model $\mu^{(t)}$, current ICNN params $\psi^{(t)}$.
\State $\mu^{(t+1)} \leftarrow \dfrac{\sum_{k=1}^q \omega_k^{(t)} \theta_k^{(t+1)}}{\sum_{j=1}^q \omega_j^{(t)}}$.
\State $\psi_0 \leftarrow \psi^{(t)}$.
\For{step $s=1$ to $S_{agg}$} \Comment{$S_{agg}$ is number of aggregation steps for $\psi$}
    \State $\psi_s \leftarrow \psi_{s-1} - \lambda_{\psi} \nabla_{\psi} \left[ \sum_{k=1}^q \frac{\omega_k^{(t)}}{\sum_{j=1}^q \omega_j^{(t)}} \mathcal{R}(\theta_k^{(t+1)}; \mu^{(t+1)}, \psi_{s-1}) \right]$.
\EndFor
\State $\psi^{(t+1)} \leftarrow \psi_{S_{agg}}$.
\State Broadcast $\mu^{(t+1)}$ and $\psi^{(t+1)}$ to all sites.
\end{algorithmic}
\end{algorithm}

\noindent
These steps of local optimisation and global aggregation are repeated iteratively until a predefined number of communication rounds is reached. The final personalised model for each site $k$ is its last updated local model $\theta_k^{(T)}$, where $T$ is the total number of rounds.

\subsubsection*{T2: Local Fine-tuning with Learned Prior}
For sites with limited networking capabilities or those unable to participate in federated rounds, T2 enables local fine-tuning using the learned global prior from FedMAP. These sites receive the pre-trained global model $\mu^{(T)}$ and ICNN parameters $\psi^{(T)}$ from completed federated training, then perform local MAP optimisation without further communication.

\begin{algorithm}[htb]
\caption{FedMAP - T2: Local Fine-tuning}
\label{algo:tier2_finetuning}
\begin{algorithmic}[1]
\State \textbf{Input:} Pre-trained global model $\mu^{(T)}$, ICNN parameters $\psi^{(T)}$ from T1, local dataset $Z_{\text{new}}$, fine-tuning epochs $e_{\text{ft}}$, learning rate $\eta_{\text{ft}}$
\State \textbf{Initialisation:}
\State $\theta^{(0)} \leftarrow \mu^{(T)}$ \Comment{Initialise with global model}
\For{epoch $t = 1$ to $e_{\text{ft}}$}
    \State Sample mini-batch $B \subset Z_{\text{new}}$
    \State Compute gradient: $g \leftarrow \nabla_\theta \left[ \frac{1}{|B|} \sum_{(x,y) \in B} \mathcal{L}(y, \phi(x; \theta^{(t-1)})) + \mathcal{R}(\theta^{(t-1)}; \mu^{(T)}, \psi^{(T)}) \right]$
    \State Update: $\theta^{(t)} \leftarrow \theta^{(t-1)} - \eta_{\text{ft}} \cdot g$
\EndFor
\State \textbf{Output:} Fine-tuned personalised model $\theta^{(e_{\text{ft}})}$
\end{algorithmic}
\end{algorithm}

\noindent
The fine-tuning process in Algorithm~\ref{algo:tier2_finetuning} leverages the learned regularisation structure from federated training, enabling effective personalisation even with limited local data. The ICNN-based regulariser $\mathcal{R}(\theta; \mu^{(T)}, \psi^{(T)})$ acts as an informative prior that captures the statistical structure learned from the federated network, preventing overfitting while allowing adaptation to local data characteristics.

\subsubsection*{T3: Inference-only}
For sites with minimal computational resources or insufficient local data, T3 provides direct access to the pre-trained global model $\mu^{(T)}$ for inference-only. This tier requires no local training or fine-tuning, making FL benefits accessible to all healthcare providers.

\begin{algorithm}[htb]
\caption{FedMAP - T3: Inference-only}
\label{algo:tier3_inference}
\begin{algorithmic}[1]
\State \textbf{Input:} Pre-trained global model $\mu^{(T)}$ from Tier 1, test dataset $Z_{\text{test}}$
\State \textbf{Inference:}
\For{each $(x, y) \in Z_{\text{test}}$}
    \State Compute prediction: $\hat{y} \leftarrow \phi(x; \mu^{(T)})$
\EndFor
\State \textbf{Output:} Predictions $\{\hat{y}\}$ for all samples in $Z_{\text{test}}$
\end{algorithmic}
\end{algorithm}

\noindent
Algorithm~\ref{algo:tier3_inference} provides a direct deployment path for resource-constrained sites. Although it does not offer personalisation benefits, it ensures that all clinical sites can access AI-driven clinical decision support developed through collaborative learning, regardless of their computational infrastructure or technical expertise.\\

\noindent
This multi-tier architecture ensures that FedMAP's benefits extend across the entire healthcare ecosystem, from well-resourced academic medical centres to community clinics with limited infrastructure.

\subsection*{Experiments}\label{sec4}
\subsubsection*{Datasets}
\textbf{Clinical Practice Research Datalink - Hospital Episode Statistics (CPRD-HES).} The CPRD~\cite{herrett_data_2015} represents one of the largest longitudinal databases of anonymised electronic health records from UK primary care, encompassing over 11.3 million patients from 674 general practitioner (GP) practices. Patients in the dataset are broadly representative of national demographics in terms of age, sex, and ethnicity.\\

\noindent
For our FL experiments, we utilised the CPRD primary care data linked with Hospital Episode Statistics (HES) and Office for National Statistics (ONS) mortality records, following the methodological framework established by Xu et al.~\cite{xu_prediction_2021} for 10-year cardiovascular disease (CVD) risk prediction.
The data includes longitudinal health information including demographics, symptoms, diagnoses, prescriptions, test results, and referrals.\\

\noindent
We adopted Xu et al.'s~\cite{xu_prediction_2021} landmark modelling approach to construct our analysis cohort, extracting risk factors at multiple landmark ages (40, 45, 50, 55, 60, 65, 70, 75, 80, and 85 years). For each landmark age, we included patients with no prior CVD diagnosis or death. Risk factors extracted at each landmark age included measurements of systolic blood pressure (SBP), total cholesterol, high-density lipoprotein (HDL) cholesterol, and binary indicators for smoking status, hypertension medication prescription, and diabetes diagnosis. The landmark datasets were subsequently stacked to create a cohort spanning multiple age groups, following Xu et al~\cite{xu_prediction_2021}.\\

\noindent
We retained only patients with complete records at each landmark age, defined as having all risk factor measurements recorded within one year prior to the landmark age and an available outcome variable. After applying this filtering criterion, the final dataset for our experiments comprised \num{258688} patients from \num{387} practices covering the period 2004-2017. Each practice was treated as an independent site for our FL setup. The median follow-up period was 8.9 years, during which CVD events were identified through combined CPRD, HES, and ONS records using the hierarchical clinical coding system employed in UK primary care (Read codes~\cite{herrett_data_2015}) and ICD-10 classifications, consistent with the QRISK~\cite{hippisley-cox_development_2017} algorithm definition~\cite{xu_prediction_2021}.\\

\noindent
\textbf{INTERVAL.} We used data from the INTERVAL~\cite{angelantonio_efficiency_2017} (ISRCTN24760606) randomised controlled trial. The analysis employed full blood count measurements obtained from Sysmex haematology analysers, following the preprocessing pipeline by Kreuter et al.~\cite{kreuter_artificial_2025} Like in their study, the binary classification outcome was iron deficiency (serum ferritin concentration $<\SI{15}{\micro\gram\per\litre}$) detection using only full blood count data. Preprocessing consisted of exclusion of unreliable measurements (suspected red cell turbidity, red cell agglutination, or platelet clumping), exclusion of samples where the delay between venepuncture and sample measurement exceeded 36 hours, and exclusion of subjects outside eligible height (\SIrange{0.5}{2.5}{\metre}) and weight (\SIrange{30}{190}{\kilo\gram}), resulting in \num{31949} subjects.\\

\noindent
We created partitions of the dataset to represent a diverse mix of blood donor populations, filtering to sites A (newly enrolled donors), B (non-european ancestry donors), C (European ancestry donors, older age distribution), and D (European ancestry donors, younger age distribution). The distributions within the partitions are illustrated in Fig.~\ref{fig:datasets}.\\

\noindent
New donors (regardless of other characteristics) were assigned to site A (\num{1662} participants, \SI{5.2}{\percent}). Non-European ancestry donors (based on self-identified ethnicity) were allocated to site B to represent centres serving ethnically diverse populations (\num{2024} participants, \SI{6.3}{\percent}). 
European ancestry donors, making up the majority of the dataset, were probabilistically assigned to age-stratified sites using a logistic function
\begin{align*}
    f(x,k,x_0) &= \frac{L}{1+e^{-k(x-x_0)}} \; ,
\end{align*}
to achieve a smooth distribution between age groups.
For site C (\num{15772} participants, \SI{49.4}{\percent}), donors were drawn with probability $p=f(\text{age},0.2,60)$, while for site D we used $p=1-f(\text{age},0.2,40)$ (\num{12491} participants, \SI{39.1}{\percent}). Donors not assigned by this strategy were discarded for our experiments. This partitioning strategy aims to reflect realistic demographic clustering patterns, where donor age and ethnicity distributions vary geographically due to local population demographics and recruitment patterns.\\

\noindent
\textbf{eICU Collaborative Research Database.} Lastly, we used the publicly available eICU Collaborative Research Database~\cite{pollard_eicu_2018}, which contains deidentified critical care data from 200,859 patient encounters across hospitals in the United States between 2014-2015. Following the preprocessing methodology of Elhussein and Gürsoy~\cite{pmlr-v219-elhussein23a}, we extracted patients with complete data across three feature domains: diagnosis codes, medications, and physiological markers recorded within the first 48 hours of ICU admission.\\

\noindent
In line with the studies we used as pre-processing guidelines, we report C-index for CPRD (Xu et al.~\cite{xu_prediction_2021}), balanced accuracy (reflecting both sensitivity and specificity equally) for INTERVAL (Kreuter et al.~\cite{kreuter_artificial_2025}), and area under the receiver operator characteristic curve (AUROC) for eICU (Elhussein and Gürsoy~\cite{pmlr-v219-elhussein23a}) as the metrics for our experiments. 

\subsubsection*{Feature skew, label skew, and quantity skew within the datasets}
Analysis revealed substantial data heterogeneity across sites for all datasets (Table~\ref{tab:dataset_characteristics}, Fig~\ref{fig:datasets}). Quantity skew manifested through highly variable sample sizes per site, with coefficients of variation ranging from 0.68 (CPRD) to 1.53 (eICU). Label distribution skew exhibited similar variability through heterogeneous positive class ratios, with coefficients of variation from 0.36 (CPRD) to 1.29 (eICU). Feature distribution skew was quantified using 1D-Wasserstein distances between sites, calculated on the most predictive feature as determined by an XGBoost~\cite{chen_xgboost_2016} classifier with default settings. The eICU dataset exhibited pronounced heterogeneity (\num{0.72\pm0.75}), indicating substantial differences in patient populations and clinical protocols.\\

\noindent
All datasets exhibited class imbalances characteristic of clinical diagnostic scenarios, with positive class ratios consistently below \SI{20}{\percent} (CPRD: \SI{14}{\percent}, INTERVAL: \SI{17}{\percent}, eICU: \SI{7}{\percent}). This configuration provides direct observation of clinical FL challenges, where multiple forms of statistical heterogeneity occur simultaneously rather than in isolation.

\begin{figure}[htbp]
    \centering
    \includegraphics[width=1\linewidth]{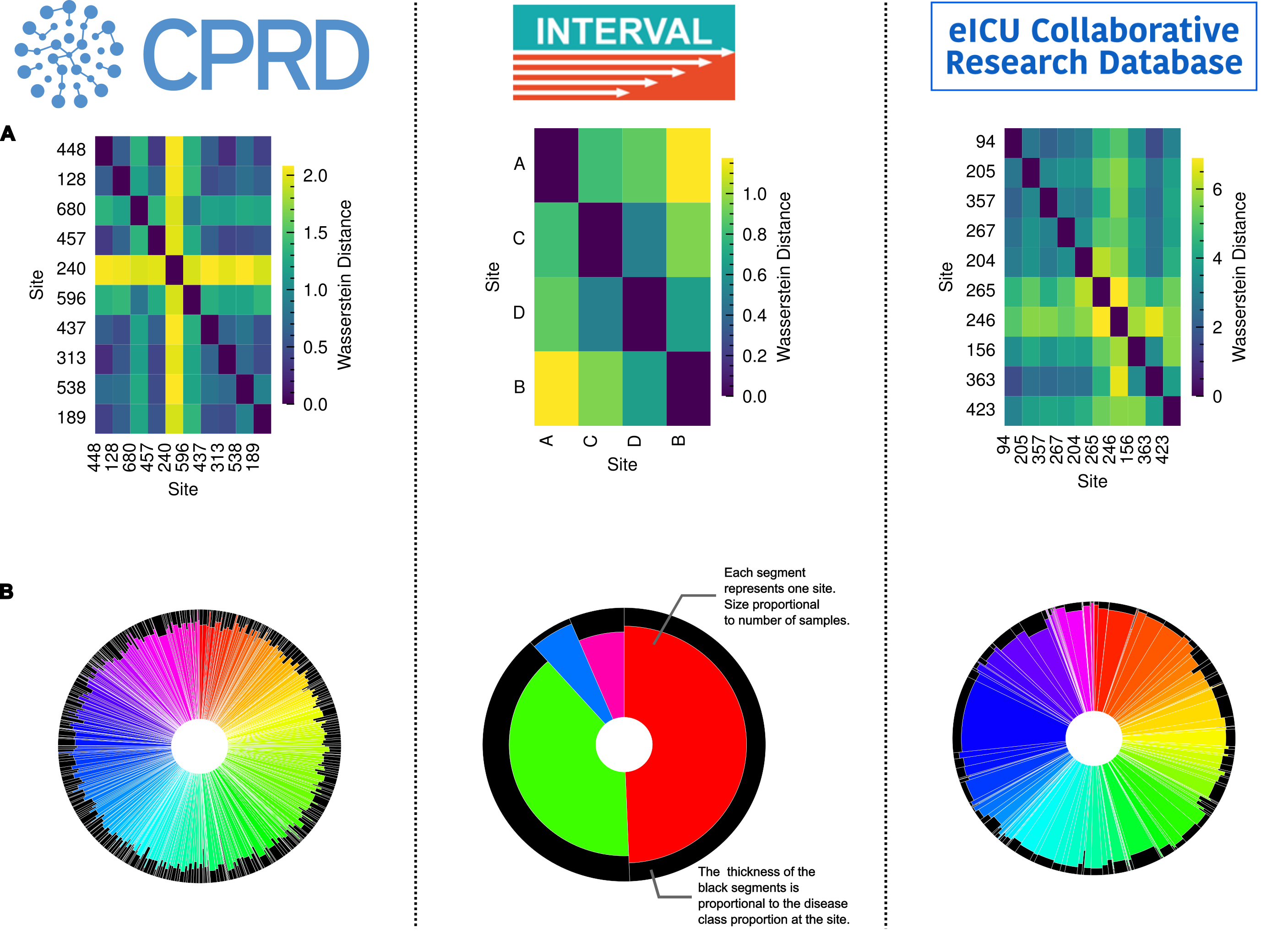}
    \caption{\textbf{Healthcare datasets.}
    We used data from the Clinical Practice Research Datalink (CPRD)~\cite{herrett_data_2015, xu_prediction_2021}, the INTERVAL trial~\cite{angelantonio_efficiency_2017}, and data from the eICU Collaborative Research Database~\cite{pollard_eicu_2018}.
    \textbf{a,} Heatmaps of Wasserstein distance between feature distributions of up to ten sites of each dataset (top ten sites with largest distance).
    \textbf{b,} Radial plot displaying sample and class distribution over the sites for the three datasets.
    }
    \label{fig:datasets}
\end{figure}

\begin{table}[htbp]
\caption{Dataset characteristics across our medical datasets}\label{tab:dataset_characteristics}%
\begin{tabular}{@{}llll@{}}
\toprule
 & \textbf{CPRD} & \textbf{INTERVAL} & \textbf{eICU}\\
\midrule
Prediction task & 10-year CVD risk & Iron deficiency & Mortality\\
$n_{\text{subjects}}$ & \num{258688} & \num{31949} & \num{44842}\\
$n_{\text{features}}$ & \num{7} & \num{16} & \num{2104}\\
$n_{\text{sites}}$ & \num{387} & \num{4} & \num{150}\\
$n_{\text{subjects}}$ per site & \num{799.2 \pm 540.6} & \num{7987.2 \pm 7221.6} & \num{298.9 \pm 456.7}\\
Positive class ratio per site & \num{0.14 \pm 0.05} & \num{0.17 \pm 0.08} & \num{0.07 \pm 0.09}\\
Wasserstein distance between sites & \num{0.40 \pm 0.31} & \num{0.38 \pm 0.25} & \num{0.72 \pm 0.75}\\
\bottomrule
\end{tabular}
\footnotetext{Numbers represent mean $\pm$ standard deviation across sites. CVD, cardiovascular disease.}
\end{table}

\clearpage
\subsubsection*{Experimental Setup}
\paragraph{Tier Grouping}
We designed a grouping strategy to simulate the three-tier deployment scenarios described in Figure~\ref{fig:fedmap}. Sites were assigned to tiers based on dataset size, which correlates with the computational resources and technical infrastructure typically available at healthcare sites of different scales. In CPRD, sites with $>2{,}000$ patients were assigned to T1 (185 practices, mean $\pm$ standard deviation: $3{,}276 \pm 1{,}192$ patients); sites with 500--2,000 patients to T2 (188 practices, $1{,}265 \pm 433$ patients); and sites with $<500$ patients to T3 (16 practices, $308 \pm 159$ patients). In eICU, sites with $>400$ patients were assigned to T1 (51 hospitals, $689 \pm 518$ patients); sites with 100--400 patients to T2 (25 hospitals, $168 \pm 60$ patients); and sites with $<100$ patients to T3 (74 hospitals, $48 \pm 53$ patients). For INTERVAL dataset, T1 included three sites with 15,772, 12,491, and 2,024 donors. T2 represented a single site with 1,361 donors. T3 simulated a small centre with 301 donors.

\paragraph{Train Test Split Strategy}
To evaluate local performance under realistic settings, data at each location were divided into separate training and validation sets, with no sharing between sites. For CPRD, each practice's data were randomly partitioned at the patient level into 80\% training and 20\% validation sets, with all landmark age observations for a given patient assigned to the same partition to prevent data leakage. In eICU, we partitioned each hospital's data into 80\% training and 20\% validation sets using stratified random sampling to maintain the balance of mortality outcomes. For INTERVAL, the three T1 sites and single T2 site were each randomly partitioned into 80\% training and 20\% validation sets following the demographic stratification defined from Kreuter et al.~\cite{kreuter_artificial_2025}. The T3 site (\num{301} donors) was used exclusively for inference-only evaluation.

\paragraph{Baseline Methods}
We compared FedMAP against established FL and PFL approaches representing different algorithmic paradigms: FedAvg~\cite{mcmahan_communication-efficient_2016}, the reference FL method aggregating model updates through weighted averaging; pFedME~\cite{t2020personalized}, a regularisation-based framework employing Moreau envelopes for personalised bi-level optimisation; pFedBayes~\cite{pmlr-v162-zhang22o}, which utilizes variational inference to learn personalised posterior distributions; pFedGP~\cite{achituve2021personalized}, modelling inter site relationships through Gaussian process priors; and pFedFDA~\cite{mclaughlin2024personalized}, addressing feature distribution skew via federated domain adaptation. Also, we evaluated individual (non-federated) training, where each site trained independently on local data only, to quantify FL performance gains.

\paragraph{Implementation Details}
FedMAP was implemented in  \texttt{PyTorch} v2.6.0 and integrated into the Flower FL framework~\cite{beutel2020flower} (v1.15) as a custom aggregation strategy. All models were trained using the Adam optimiser with a learning rate of 0.001, early stopping with a patience of 5 epochs, and dropout regularisation. For the cardiovascular disease risk prediction task using CPRD data, we used the DeepSurv~\cite{katzman_deepsurv_2018}-based Cox proportional hazards model described in Appendix Table~\ref{tab:model-architecture-cprd}. Training was performed with a batch size of 128, dropout rate of 0.1, 10 local epochs per round, and 100 communication rounds. For the mortality prediction task on eICU data, we used the multimodal feedforward neural network described in Appendix Table~\ref{tab:model-architecture-eicu}. The model was trained with a batch size of 32, dropout rate of 0.3, weight decay of $10^{-5}$, 5 local epochs per round, and 30 communication rounds. For the iron deficiency detection task on INTERVAL data, we used the Multi-layer Perceptron (MLP) classifier described in Appendix Table~\ref{tab:model-architecture-interval-expanded}. The model was trained with cross-entropy loss, batch size 128, weight decay $10^{-5}$, 2 local epochs per round, and 50 communication rounds.  Source code is publicly available at \url{https://github.com/CambridgeCIA/FedMAP}. \\

\subsubsection*{Computational Resources}
All experiments were conducted on a high performance computing cluster running Rocky Linux 8.10 (Green Obsidian). The system featured two Intel Xeon (Sapphire Rapids) processors with a total of 188 logical CPUs operating at 2.1 GHz frequency, supported by 1.9 TiB of RAM.


\clearpage
\backmatter
\section*{Declarations}

\bmhead{Funding}
F.Z, D.K, C.-B.S, N.S.G., M.R. have received support from the Trinity Challenge grant awarded to establish the BloodCounts! consortium, along with NIHR UCLH Biomedical Research Centre, the NIHR Cambridge Biomedical Research Centre, National Health Service Blood and Transplant (NHSBT) and the Medical Research Council. S.I. was supported by Cancer Research UK EDDPMA-May22\textbackslash{}100062 and by Cancer Research UK (EDDAPA-2024/100011). A.W is supported by the BHF Data Science Centre (HDRUK2023.0239), Health Data Research UK (Big Data for Complex Disease-HDR-23012), and as an NIHR Research Professor (NIHR303137). S.I and A.M.W were supported by core funding from the British Heart Foundation (RG/F/23/110103), NIHR Cambridge Biomedical Research Centre (NIHR203312)\footnote{The views expressed are those of the authors and not necessarily those of the NIHR or the Department of Health and Social Care.}. BHF Chair Award (CH/12/2/29428), and by Health Data Research UK, which is funded by the UK Medical Research Council, Engineering and Physical Sciences Research Council, Economic and Social Research Council, Department of Health and Social Care (England), Chief Scientist Office of the Scottish Government Health and Social Care Directorates, Health and Social Care Research and Development Division (Welsh Government), Public Health Agency (Northern Ireland), British Heart Foundation and the Wellcome Trust. C.E.-Y. is supported by the Ramón y Cajal 2022 grant RYC2022-035966-I. J.H.F.R. is part-supported by the NIHR Cambridge Biomedical Research Centre, the British Heart Foundation Centre of Research Excellence (RE/24/130011) and the Wellcome Trust. N.S.G. has been supported by NHSBT grants 1701-GEN; 20-01-GEN; G120400. C.-B.S. acknowledges support from the EPSRC programme grant in 'The Mathematics of Deep Learning' (EP/L015684), Cantab Capital Institute for the Mathematics of Information, the Philip Leverhulme Prize, the Royal Society Wolfson Fellowship, the EPSRC grants EP/S026045/1 and EP/T003553/1, EP/N014588/1, EP/T017961/1, the Wellcome Innovator Award RG98755 and the Alan Turing Institute. M.R. is additionally supported by the British Heart Foundation (TA/F/20/210001).

\bmhead{Competing interests}
All authors declare that they have no conflicts of interest.

\bmhead{Ethics approval and consent to participate}
The INTERVAL trial is compiled into the Blood Donors Studies BioResource (BDSB) with Research Ethics Committee (REC) reference 20/EE/0115.
The Clinical Practice Research Datalink (CPRD) is approved by the UK Health Research Authority for observational research. The data comprise anonymised patient records collected by the NHS as part of routine clinical care. Individual patient consent is not required for studies using CPRD data under the terms of the CPRD Research Data Governance Process, which operates within UK legislation and guidance on the use of patient data for research. The study protocol was approved by the Independent Scientific Advisory Committee for MHRA Database Research.
The eICU study is exempt from institutional review board approval due to the retrospective design, lack of direct patient intervention, and the security schema, for which the re-identification risk was certified as meeting safe harbor standards by an independent privacy expert (Privacert, Cambridge, MA, USA) (Health Insurance Portability and Accountability Act Certification no. 1031219-2).

\bmhead{Consent for publication}
All authors have reviewed and approved the submission.

\bmhead{Data availability}
Access to BDSB data can be applied for from the Blood Donors Studies Data Access Committee.
The CPRD data is under licence from the UK Medicines and Healthcare Products Regulatory Agency. CPRD data are available to researchers through application to CPRD (\url{https://www.cprd.com/access-data}). Due to data governance restrictions, the data cannot be made publicly available. Researchers wishing to access similar data should apply directly to CPRD.
The eICU database is publicly available to researchers who complete the required training and credentialing process. Access to the database can be requested through PhysioNet at \url{https://eicu-crd.mit.edu/}. Researchers must complete the CITI "Data or Specimens Only Research" training course and sign a data use agreement to gain access. Further information about the database and access procedures can be found at \url{https://eicu-crd.mit.edu/gettingstarted/access/}.


\bmhead{Code availablity}
Source code is publicly available at \url{https://github.com/CambridgeCIA/FedMAP}.

\bmhead{Author contributions}
F.Z., C.E.-Y., and S.D. conceptualised the study and designed the methodology. F.Z. and D.K. performed the experiments and analysed the results. F.Z., D.K., and C.E.-Y. drafted the manuscript. M.R. provided supervision.
\backmatter


\bmhead{Acknowledgements}
Participants in the INTERVAL randomised controlled trial were recruited with the active collaboration of NHS Blood and Transplant England (\url{https://www.nhsbt.nhs.uk}), which has supported field work and other elements of the trial. DNA extraction and genotyping were co-funded by the National Institute for Health and Care Research (NIHR), the NIHR BioResource (\url{http://bioresource.nihr.ac.uk}) and the NIHR Cambridge Biomedical Research Centre (BRC-1215-20014)\footnote{The views expressed are those of the authors and not necessarily those of the NIHR or the Department of Health and Social Care.}.

\medskip

\noindent The academic coordinating centre for INTERVAL was supported by core funding from the: NIHR Blood and Transplant Research Unit (BTRU) in Donor Health and Genomics (NIHR BTRU-2014-10024), NIHR BTRU in Donor Health and Behaviour (NIHR203337), UK Medical Research Council (MR/L003120/1), British Heart Foundation (SP/09/002; RG/13/13/30194; RG/18/13/33946), NIHR Cambridge BRC (BRC-1215-20014; NIHR203312), and by Health Data Research UK, which is funded by the UK Medical Research Council, Engineering and Physical Sciences Research Council, Economic and Social Research Council, Department of Health and Social Care (England), Chief Scientist Office of the Scottish Government Health and Social Care Directorates, Health and Social Care Research and Development Division (Welsh Government), Public Health Agency (Northern Ireland), British Heart Foundation and Wellcome.

\medskip

\noindent A complete list of the investigators and contributors to the INTERVAL trial is provided in reference~\citep{angelantonio_efficiency_2017}. The academic coordinating centre would like to thank blood donor centre staff and blood donors for participating in the INTERVAL trial.

\bmhead{BloodCounts! Consortium}
Martijn Schut$^{17}$, Folkert Asselbergs$^{17}$, Sujoy Kar$^{18}$, Suthesh Sivapalaratnam$^{19}$, Sophie Williams$^{19}$, Mickey Koh$^{20}$, Yvonne Henskens$^{21}$, Norbert C.J. de Wit$^{21}$, Umberto D'Alessandro$^{22}$, Bubacarr Bah$^{22}$, Ousman Secka$^{22}$, Parashkev Nachev$^{23}$, Rajeev Gupta$^{23}$, Sara Trompeter$^{23}$, Nancy Boeckx$^{24}$, Christine van Laer$^{24}$, Gordon. A. Awandare$^{25}$, Kwabena Sarpong$^{25}$, Lucas Amenga-Etego$^{25}$, Mathie Leers$^{26}$, Mirelle Huijskens$^{26}$, Samuel McDermott$^{1}$, Willem H. Ouwehand$^{8}$, James Rudd$^{9}$, Carola-Bibiane Schönlieb$^{1}$, Nicholas Gleadall$^{3,7,8}$ and Michael Roberts$^{1,9}$.\\

\noindent
${}^{17}$ Amsterdam University Medical Centre, Amsterdam, Netherlands.
${}^{18}$ Apollo Hospitals, Chennai, India.
${}^{19}$ Barts Health NHS Trust, London, United Kingdom.
${}^{20}$ Health Services Authority, Singapore.
${}^{21}$ Maastricht University Medical Centre, Maastricht, Netherlands.
${}^{22}$ MRC The Gambia Unit, Banjul, The Gambia.
${}^{23}$ University College London Hospitals, London, United Kingdom.
${}^{24}$ University Hospitals Leuven, Leuven, Belgium.
${}^{25}$ West African Centre for Cell Biology of Infectious Pathogens, Accra, Ghana.
${}^{26}$ Zuyderland Medical Center, Zuyderland, Netherlands.
\clearpage
\bibliography{references}

\clearpage
\begin{appendices}

\section{Proof of Theorem \ref{thm:main paper}}\label{secA1}

In this section we provide the proof of Theorem \ref{thm:main paper}. 
Let us recall the convexity assumption made in Theorem \ref{thm:main paper}, concerning the parameter spaces $\Theta$ and $\Gamma$, the loss functional $\mathcal{L} (\theta; Z_k)$, and  the parametric regulariser $\mathcal{R} (\theta; \gamma)$.
\begin{assumption}
\label{assum: convexity}
Throughout this section, we assume:
\begin{enumerate}
    \item The parameter space $\Theta \subset \mathbb{R}^d$ is compact and convex.
    \item The parameter space for the regulariser $\Gamma \subset \mathbb{R}^p$ is convex.
    \item For each $k = 1, \ldots , q$, the function $\theta \mapsto \mathcal{L} (\theta; Z_k)$ is continuous and convex in $\Theta$.
    \item The function $(\theta, \gamma) \mapsto \mathcal{R} (\theta; \gamma)$ is differentiable and strongly convex in $\Theta \times \Gamma$, i.e., there exists $\alpha_1 > 0$ such that for all $(\theta_1, \gamma_1), (\theta_2, \gamma_2)\in \Theta \times \Gamma$:
    \begin{equation}
    \label{strong convexity}
    \dfrac{1}{2} \mathcal{R} (\theta_1; \gamma_1) + \dfrac{1}{2} \mathcal{R} (\theta_2; \gamma_2) \geq \mathcal{R} \left( \dfrac{\theta_1 + \theta_2}{2}; \dfrac{\gamma_1+ \gamma_2}{2}  \right) + \alpha_1 \left( \left\| \dfrac{\theta_1 - \theta_2}{2}  \right\|^2 + \left\| \dfrac{\gamma_1 - \gamma_2}{2}  \right\|^2 \right).
    \end{equation}
\end{enumerate}
\end{assumption}

\noindent
Assumption \ref{assum: convexity} implies that, for any $\gamma\in \Gamma$, the minimiser in the right-hand-side of \eqref{MAP estimation} is attained at a unique $\theta^\ast\in \Theta$, which obviously depends on $\gamma$.

\noindent
We start the proof of Theorem \ref{thm:main paper} by proving that the function $M(\gamma)$ defined in \eqref{M of gamma def} is strongly convex in $\Gamma$. This corresponds to the second point in Theorem \ref{thm:main paper}, and guarantees that with a sufficiently small learning rate $\lambda$, the FedMAP iterates for $\gamma^{(t)}$ converge to the unique minimiser of $M(\gamma)$.

\begin{lemma}
\label{lem:M_convex}
Under the hypotheses in Assumption \ref{assum: convexity}, the function $M(\gamma)$ defined in \eqref{M of gamma def} is strongly convex in $\Gamma$. That is, there exists $\alpha_M > 0$ such that
$$
\dfrac{1}{2} M(\gamma_1) + \dfrac{1}{2} M(\gamma_2) \geq M \left( \dfrac{\gamma_1 + \gamma_2}{2} \right) + \alpha_M 
\left\| \dfrac{\gamma_1 - \gamma_2}{2} \right\|^2,
\qquad
\forall\gamma_1,\gamma_2 \in \Gamma.
$$
\end{lemma}

\begin{proof}
    Let $\gamma_1, \gamma_2 \in \Gamma$. For each $k=1, \ldots q$, let $\theta_{k,i}^\ast = \theta_k^\ast(\gamma_i)$ for $i=1,2$ be the unique minimisers of $\theta \mapsto \mathcal{L}(\theta; Z_k) + \mathcal{R}(\theta; \gamma_i)$.
    By definition of $M_k(\gamma; Z_k)$ in \eqref{M of gamma def}, we have:
    $$
    M_k (\gamma_i ; Z_k) = \mathcal{L} (\theta_{k,i}^\ast; Z_k) + N_k\mathcal{R} (\theta_{k,i}^\ast; \gamma_i), \qquad \text{for} \ i = 1,2.
    $$

    \noindent
    Now, consider $M_k \left( \frac{\gamma_1 + \gamma_2}{2} ; Z_k \right)$. By the convexity of $\mathcal{L}$ and $\mathcal{R}$ in Assumption \ref{assum: convexity}, we have
    \begin{align*}
         M_k \left( \dfrac{\gamma_1+ \gamma_2}{2}; Z_k  \right) & \leq \mathcal{L} \left( \dfrac{\theta_{k,1}^\ast+ \theta_{k,2}^\ast}{2} ; Z_k  \right)  + N_k \mathcal{R} \left( \dfrac{\theta_{k,1}^\ast + \theta_{k,2}^\ast}{2} ;\dfrac{\gamma_1+ \gamma_2}{2}  \right) \\
         & \leq \dfrac{1}{2} \mathcal{L} (\theta_{k,1}^\ast; Z_k) + \dfrac{1}{2} \mathcal{L} (\theta_{k,2}^\ast; Z_k) + \dfrac{N_k}{2} \mathcal{R} (\theta_{k,1}^\ast; \gamma_1) + \dfrac{N_k}{2} \mathcal{R} (\theta_{k,2}^\ast; \gamma_2) \\
         & \quad - \alpha_1 \left( \left\| \dfrac{ \theta_{k,1}^\ast - \theta_{k,2}^\ast}{2} \right\|^2 +  \left\| \dfrac{ \gamma_1 - \gamma_2}{2} \right\|^2  \right).
    \end{align*}
    Re-arranging the terms we obtain
    $$
    \dfrac{1}{2} M_k (\gamma_1; Z_k) + \dfrac{1}{2} M_k (\gamma_2; Z_k) \geq M_k \left( \dfrac{\gamma_1+ \gamma_2}{2} ; Z_k  \right) + \alpha_1 \left\| \dfrac{\gamma_1 - \gamma_2}{2} \right\|^2 ,
    $$
    and summing over $k=1,\ldots, q$, we obtain
    $$
    \dfrac{1}{2} M(\gamma_1) + \dfrac{1}{2} M(\gamma_2) \geq M \left( \dfrac{\gamma_1 + \gamma_2}{2} \right) + q\alpha_1 \left\| \dfrac{\gamma_1 - \gamma_2}{2} \right\|^2.
    $$
    Finally, we take $\alpha_M = q\alpha_1$ and obtain that
    $$
    \dfrac{1}{2} M(\gamma_1) + \dfrac{1}{2} M(\gamma_2) \geq M \left( \dfrac{\gamma_1 + \gamma_2}{2} \right) + \alpha_M \left\|  \dfrac{\gamma_1 - \gamma_2}{2} \right\|^2,
    $$
    thus concluding the proof of strong convexity for $M(\gamma)$.
\end{proof}

\begin{proof}[Proof of Theorem \ref{thm:main paper}]
    Due to Assumption \ref{assum: convexity}, for each $k\in \{1, \ldots, q\}$ and for any $\gamma\in \Gamma$,
    the function $\theta \mapsto \mathcal{L} (\theta; Z_k) + \mathcal{R} (\theta;\gamma)$ is strongly convex in $\Theta$, and therefore, there exists a unique minimiser $\theta_k^\ast\in \Theta$, which gives $M_k (\gamma ; Z_k) =  \mathcal{L} (\theta_k^\ast; Z_k) + \mathcal{R} (\theta_k^\ast;\gamma)$.\\

    \noindent
    Using the definition of $M_k (\gamma ; Z_k)$, we can apply Danskin's Theorem \cite{danskin_theory_1966,bernhard_theorem_1995} to obtain
    $$
    \nabla_\gamma M_k (\gamma; Z_k) = \nabla_\gamma \left[ \mathcal{L} (\theta_k^\ast; Z_k) + N_k \mathcal{R} (\theta_k^\ast; \gamma) \right],
    $$
    where $\theta_k^\ast$ is the unique minimiser of $\theta \mapsto \mathcal{L} (\theta; Z_k) + N_k \mathcal{R} (\theta; \gamma)$.\\

    \noindent
    Since $\mathcal{L}(\theta; Z_k)$ is independent of $\gamma$, the gradient of $M_k (\gamma ; Z_k)$ with respect to $\gamma$ is given by
    $$
    \nabla_\gamma M_k (\gamma; Z_k) = N_k \nabla_\gamma\mathcal{R} (\theta_k^\ast; \gamma).
    $$

    \noindent
    Hence, the gradient of the function $M(\gamma)$ in \eqref{M of gamma def} is given by
    $$
    \nabla_\gamma M (\gamma) = \sum_{k=1}^q N_k \nabla_\gamma \mathcal{R} (\theta_k^\ast; \gamma),
    $$
    where for each $k=1, \ldots, q$, $\theta_k^\ast$ is the unique minimiser of $\theta\mapsto \mathcal{L} (\theta; Z_k) + N_k \mathcal{R} (\theta; \gamma)$.
    
    Therefore, the update rule \eqref{prior MLE} can be written as
    $$
    \gamma^{(t+1)} = \gamma^{(t)} - \lambda \sum_{k=1}^q N_k \nabla_\gamma \mathcal{R} (\theta_k^{(t)}; \gamma^{(t)}) = \gamma^{(t)} - \lambda \nabla_\gamma M(\gamma^{(t)}),
    $$
    where for each $k=1, \ldots, q$, $\theta_k^{(t)}$ is the unique minimiser of $\theta\mapsto \mathcal{L} (\theta; Z_k) + N_k \mathcal{R} (\theta; \gamma^{(t)})$.
    
    \noindent
    This proves that the FedMAP iterations (Equations \eqref{MAP estimation} and \eqref{prior MLE}) effectively correspond to gradient descent iterations for $M(\gamma)$.
    The strong convexity of $M(\gamma)$ has been proved in Lemma \ref{lem:M_convex}. This strong convexity ensures that $M(\gamma)$ has a unique minimiser $\gamma^\ast$.\\

    \noindent
    Now, let us prove that the minimiser $\gamma^\ast$ of $M(\gamma)$, together with $\theta_k^\ast := \theta_k^{\ast}(\gamma^\ast)$ defined in Theorem \ref{thm:main paper}, constitutes the unique solution to the bi-level optimisation problem \eqref{bi-level optimisation}.
    By the compactness of $\Theta$ and the strong convexity of $\mathcal{R} (\theta; \gamma)$ with respect to $(\theta, \gamma)$, the overall objective in \eqref{bi-level optimisation} is bounded from below. This guarantees the existence of a minimising sequence $\{ (\theta_1^{(n)}, \ldots, \theta_q^{(n)}, \gamma^{(n)}) \}_{n\in \mathbb{N}}$. Due to the compactness of $\Theta$, the sequence  $\{ (\theta_1^{(n)}, \ldots, \theta_q^{(n)}) \}_{n\in \mathbb{N}}$ converges (through a subsequence) to some $(\theta_1^\ast, \ldots , \theta_q^\ast)\in \Theta^q$. Also, since $\mathcal{R}$ is strongly convex, and therefore coercive, the sequence $\{  \gamma^{(n)} \}_{n\in \mathbb{N}}$ is bounded, which implies that it converges (through a subsequence) to some $\gamma^\ast\in \Gamma$. By continuity of the functions, this limit point $(\theta_1^\ast, \ldots , \theta_q^\ast, \gamma^\ast)\in \Theta^q\times \Gamma$ is a solution to the bi-level optimisation problem \eqref{bi-level optimisation}.\\

    \noindent
    To prove uniqueness, let $(\theta_1^\ast, \ldots , \theta_q^\ast, \gamma^\ast)$ be any solution to \eqref{bi-level optimisation}.
    By the first-order optimality condition for the inner minimisation problem in \eqref{bi-level optimisation}, we have:
    $$
    \sum_{k=1}^qN_k \nabla_\gamma \mathcal{R} (\theta_k^\ast; \gamma^\ast) = 0,
    $$
    Since $\theta_k^\ast$ is the unique minimiser of $\theta \mapsto \mathcal{L} (\theta; Z_k) + N_k \mathcal{R} (\theta; \gamma^\ast)$, it implies that $\gamma^\ast$ is a critical point of $M(\gamma)$. As $M(\gamma)$ is strongly convex, it has a unique critical point, which must be its unique global minimiser. Therefore, any solution to the bi-level optimisation problem corresponds to the unique minimiser of $M(\gamma)$, establishing the uniqueness of the solution to \eqref{bi-level optimisation}.
\end{proof}

\section{Model Architectures}

The neural network architectures for each clinical prediction task were designed following established methodological frameworks from the literature. For CPRD cardiovascular risk prediction, we adopted the DeepSurv architecture~\cite{katzman_deepsurv_2018}, a Cox proportional hazards neural network that has demonstrated strong performance on survival analysis tasks. For eICU mortality prediction, we implemented the multimodal architecture from Elhussein and G\"ursoy~\cite{pmlr-v219-elhussein23a}, which processes medication, diagnosis, and physiological features through separate encoding pathways before fusion. For INTERVAL iron deficiency detection, we used the MLP architecture from Kreuter et al.~\cite{kreuter_artificial_2025}. These architectures were selected to maintain consistency with prior work on these datasets, which allows direct comparison of FL methods and controlling for model design choices. \\

\noindent
The following tables describe the neural network architectures used for each of the clinical prediction tasks.

\begin{table}[h]
    \caption{MLP architecture for 10-year cardiovascular disease risk prediction with the CPRD dataset.}
    \label{tab:model-architecture-cprd}
    \centering
    \begin{tabular}{|c|p{10cm}|}
        \hline
        \rowcolor{gray!25} \textbf{Layer} & \textbf{Details} \\
        \hline
        1 & \texttt{Linear(in\_features, 64)}, ReLU, \texttt{BatchNorm1d(64)}, \texttt{Dropout(p=0.1)} \\
        \hline
        2 & \texttt{Linear(64, 32)}, ReLU, \texttt{BatchNorm1d(32)}, \texttt{Dropout(p=0.1)} \\
        \hline
        3 & \texttt{Linear(32, 16)}, ReLU, \texttt{BatchNorm1d(16)}, \texttt{Dropout(p=0.1)} \\
        \hline
        4 & \texttt{Linear(16, 1)} \\
        \hline
        5 & Output Layer (Log Hazard Ratio) \\
        \hline
    \end{tabular}
    \vspace{1em}
\end{table}

\begin{table}[h]
    \caption{Multi-modal MLP architecture for ICU mortality prediction with the eICU dataset.}
    \label{tab:model-architecture-eicu}
    \centering
    \begin{tabular}{|c|p{10cm}|}
        \hline
        \rowcolor{gray!25} \textbf{Layer} & \textbf{Details} \\
        \hline
        \multicolumn{2}{|l|}{\textbf{Medication Feature Extractor}} \\
        \hline
        1 & \texttt{Linear(input\_dim\_drugs, 100)}, ReLU \\
        \hline
        2 & \texttt{Linear(100, 50)}, ReLU \\
        \hline
        3 & \texttt{Linear(50, 10)}, ReLU \\
        \hline
        4 & \texttt{Linear(10, 5)} \\
        \hline
        \multicolumn{2}{|l|}{\textbf{Diagnosis Feature Extractor}} \\
        \hline
        5 & \texttt{Linear(input\_dim\_dx, 100)}, ReLU \\
        \hline
        6 & \texttt{Linear(100, 50)}, ReLU \\
        \hline
        7 & \texttt{Linear(50, 10)}, ReLU \\
        \hline
        8 & \texttt{Linear(10, 5)} \\
        \hline
        \multicolumn{2}{|l|}{\textbf{Physiology Feature Extractor}} \\
        \hline
        9 & \texttt{Linear(input\_dim\_physio, 40)}, ReLU \\
        \hline
        10 & \texttt{Linear(40, 20)}, ReLU \\
        \hline
        11 & \texttt{Linear(20, 10)}, ReLU \\
        \hline
        12 & \texttt{Linear(10, 5)} \\
        \hline
        \multicolumn{2}{|l|}{\textbf{Concatenation and Final Prediction}} \\
        \hline
        13 & Concatenate(Medication, Diagnosis, and Physiology outputs) \\
        \hline
        14 & \texttt{Linear(15, 15)}, ReLU \\
        \hline
        15 & \texttt{Linear(15, 10)}, ReLU \\
        \hline
        16 & \texttt{Linear(10, 5)}, ReLU \\
        \hline
        17 & \texttt{Linear(5, 2)} \\
        \hline
        18 & Output Layer \\
        \hline
    \end{tabular}
    \vspace{1em}
\end{table}

\begin{table}[h]
    \caption{MLP architecture for iron deficiency prediction with the INTERVAL dataset (expanded view).}
    \label{tab:model-architecture-interval-expanded}
    \centering
    \begin{tabular}{|c|p{10cm}|}
        \hline
        \rowcolor{gray!25} \textbf{Layer} & \textbf{Details} \\
        \hline
        \multicolumn{2}{|l|}{\textbf{Hidden Block 1}} \\
        \hline
        1 & \texttt{Linear(input\_dim, 64)} \\
        \hline
        2 & ReLU \\
        \hline
        3 & \texttt{BatchNorm1d(64)} \\
        \hline
        4 & \texttt{Dropout(p=dropout)} \\
        \hline
        \multicolumn{2}{|l|}{\textbf{Hidden Block 2}} \\
        \hline
        5 & \texttt{Linear(64, 64)} \\
        \hline
        6 & ReLU \\
        \hline
        7 & \texttt{BatchNorm1d(64)} \\
        \hline
        8 & \texttt{Dropout(p=dropout)} \\
        \hline
        \multicolumn{2}{|l|}{\textbf{Output Layer}} \\
        \hline
        9 & \texttt{Linear(64, output\_dim)} \\
        \hline
    \end{tabular}
    \vspace{1em}
\end{table}

\clearpage
\section{Supplementary Results}
\subsection{Statistical comparisons of FedMAP against other methods in FL, fine-tuning, and inference experiments}
\begin{table}[hbt]
\caption{Statistical comparison of FedMAP against baseline methods on CPRD dataset (FL scenario)}\label{tab:comparison_cprd_fl}%
\begin{tabular}{@{}lrrrrrrr@{}}
\toprule
\textbf{Method} & \textbf{N Sites} & \textbf{FedMAP} & \textbf{Baseline} & \textbf{Median Diff.} & \textbf{95\% CI} & \textbf{p-value} & \textbf{Prop. Improved}\\
\midrule
FedAvg & \num{185} & \num{0.717} & \num{0.667} & +\num{0.050} & [\num{0.045}, \num{0.055}] & \num{<0.001} & \num{0.96}\\
pFedGP & \num{185} & \num{0.717} & \num{0.691} & +\num{0.025} & [\num{0.016}, \num{0.034}] & \num{<0.001} & \num{0.70}\\
pFedME & \num{185} & \num{0.717} & \num{0.696} & +\num{0.019} & [\num{0.015}, \num{0.022}] & \num{<0.001} & \num{0.82}\\
Individual & \num{185} & \num{0.717} & \num{0.690} & +\num{0.024} & [\num{0.020}, \num{0.028}] & \num{<0.001} & \num{0.79}\\
pFedBayes & \num{178} & \num{0.717} & \num{0.692} & +\num{0.023} & [\num{0.019}, \num{0.027}] & \num{<0.001} & \num{0.84}\\
pFedFDA & \num{185} & \num{0.717} & \num{0.502} & +\num{0.205} & [\num{0.179}, \num{0.228}] & \num{<0.001} & \num{0.99}\\
\bottomrule
\end{tabular}
\footnotetext{Median performance metrics across healthcare sites. CI, confidence interval; Prop., Proportion.}
\end{table}

\begin{table}[hbt]
\caption{Statistical comparison of FedMAP against baseline methods on CPRD dataset (fine-tuning scenario)}\label{tab:comparison_cprd_ft}%
\begin{tabular}{@{}lrrrrrrr@{}}
\toprule
\textbf{Method} & \textbf{N Sites} & \textbf{FedMAP} & \textbf{Baseline} & \textbf{Median Diff.} & \textbf{95\% CI} & \textbf{p-value} & \textbf{Prop. Improved}\\
\midrule
FedAvg & \num{188} & \num{0.708} & \num{0.663} & +\num{0.045} & [\num{0.033}, \num{0.054}] & \num{<0.001} & \num{0.82}\\
pFedGP & \num{188} & \num{0.708} & \num{0.692} & +\num{0.016} & [\num{0.003}, \num{0.030}] & \num{0.064} & \num{0.59}\\
pFedME & \num{188} & \num{0.708} & \num{0.663} & +\num{0.052} & [\num{0.043}, \num{0.062}] & \num{<0.001} & \num{0.83}\\
Individual & \num{138} & \num{0.710} & \num{0.660} & +\num{0.056} & [\num{0.032}, \num{0.072}] & N/A & \num{0.75}\\
pFedBayes & \num{188} & \num{0.708} & \num{0.681} & +\num{0.029} & [\num{0.022}, \num{0.037}] & \num{<0.001} & \num{0.73}\\
pFedFDA & \num{188} & \num{0.708} & \num{0.499} & +\num{0.192} & [\num{0.172}, \num{0.213}] & \num{<0.001} & \num{0.93}\\
\bottomrule
\end{tabular}
\footnotetext{Median performance metrics across healthcare sites. CI, confidence interval; Prop., Proportion.}
\end{table}

\begin{table}[hbt]
\caption{Statistical comparison of FedMAP against baseline methods on CPRD dataset (inference-only scenario)}\label{tab:comparison_cprd_inf}%
\begin{tabular}{@{}lrrrrrrr@{}}
\toprule
\textbf{Method} & \textbf{N Sites} & \textbf{FedMAP} & \textbf{Baseline} & \textbf{Median Diff.} & \textbf{95\% CI} & \textbf{p-value} & \textbf{Prop. Improved}\\
\midrule
FedAvg & \num{14} & \num{0.670} & \num{0.636} & +\num{0.048} & [\num{-0.018}, \num{0.105}] & \num{0.863} & \num{0.64}\\
pFedGP & \num{14} & \num{0.670} & \num{0.416} & +\num{0.209} & [\num{-0.013}, \num{0.352}] & \num{0.247} & \num{0.71}\\
pFedME & \num{14} & \num{0.670} & \num{0.667} & +\num{0.005} & [\num{-0.007}, \num{0.017}] & \num{1.000} & \num{0.64}\\
pFedBayes & \num{14} & \num{0.670} & \num{0.660} & +\num{0.001} & [\num{-0.003}, \num{0.016}] & \num{1.000} & \num{0.57}\\
pFedFDA & \num{14} & \num{0.670} & \num{0.484} & +\num{0.199} & [\num{0.072}, \num{0.260}] & \num{0.123} & \num{0.86}\\
\bottomrule
\end{tabular}
\footnotetext{Median performance metrics across healthcare sites. CI, confidence interval; Prop., Proportion.}
\end{table}

\begin{table}[hbt]
\caption{Statistical comparison of FedMAP against baseline methods on eICU dataset (FL scenario)}\label{tab:comparison_eICU_fl}%
\begin{tabular}{@{}lrrrrrrr@{}}
\toprule
\textbf{Method} & \textbf{N Sites} & \textbf{FedMAP} & \textbf{Baseline} & \textbf{Median Diff.} & \textbf{95\% CI} & \textbf{p-value} & \textbf{Prop. Improved}\\
\midrule
FedAvg & \num{51} & \num{0.808} & \num{0.787} & +\num{0.019} & [\num{0.001}, \num{0.043}] & \num{0.057} & \num{0.65}\\
pFedGP & \num{51} & \num{0.808} & \num{0.732} & +\num{0.053} & [\num{0.033}, \num{0.080}] & \num{<0.001} & \num{0.80}\\
pFedME & \num{51} & \num{0.808} & \num{0.772} & +\num{0.033} & [\num{0.010}, \num{0.055}] & \num{<0.001} & \num{0.71}\\
Individual & \num{51} & \num{0.808} & \num{0.701} & +\num{0.098} & [\num{0.066}, \num{0.137}] & \num{<0.001} & \num{0.86}\\
pFedBayes & \num{51} & \num{0.808} & \num{0.769} & +\num{0.049} & [\num{0.020}, \num{0.071}] & \num{<0.001} & \num{0.71}\\
pFedFDA & \num{51} & \num{0.808} & \num{0.738} & +\num{0.063} & [\num{0.029}, \num{0.085}] & \num{<0.001} & \num{0.80}\\
\bottomrule
\end{tabular}
\footnotetext{Median performance metrics across healthcare sites. CI, confidence interval; Prop., Proportion.}
\end{table}

\begin{table}[hbt]
\caption{Statistical comparison of FedMAP against baseline methods on eICU dataset (fine-tuning scenario)}\label{tab:comparison_eICU_ft}%
\begin{tabular}{@{}lrrrrrrr@{}}
\toprule
\textbf{Method} & \textbf{N Sites} & \textbf{FedMAP} & \textbf{Baseline} & \textbf{Median Diff.} & \textbf{95\% CI} & \textbf{p-value} & \textbf{Prop. Improved}\\
\midrule
FedAvg & \num{25} & \num{0.778} & \num{0.704} & +\num{0.077} & [\num{0.045}, \num{0.132}] & \num{0.001} & \num{0.92}\\
pFedGP & \num{25} & \num{0.778} & \num{0.725} & +\num{0.071} & [\num{0.037}, \num{0.121}] & \num{0.005} & \num{0.76}\\
pFedME & \num{25} & \num{0.778} & \num{0.725} & +\num{0.045} & [\num{0.020}, \num{0.069}] & \num{0.012} & \num{0.84}\\
Individual & \num{25} & \num{0.778} & \num{0.694} & +\num{0.049} & [\num{-0.013}, \num{0.098}] & \num{0.309} & \num{0.68}\\
pFedBayes & \num{25} & \num{0.778} & \num{0.753} & +\num{0.038} & [\num{0.018}, \num{0.066}] & \num{0.040} & \num{0.80}\\
pFedFDA & \num{25} & \num{0.778} & \num{0.705} & +\num{0.118} & [\num{-0.005}, \num{0.194}] & \num{0.015} & \num{0.68}\\
\bottomrule
\end{tabular}
\footnotetext{Median performance metrics across healthcare sites. CI, confidence interval; Prop., Proportion.}
\end{table}

\begin{table}[hbt]
\caption{Statistical comparison of FedMAP against baseline methods on eICU dataset (inference-only scenario)}\label{tab:comparison_eICU_inf}%
\begin{tabular}{@{}lrrrrrrr@{}}
\toprule
\textbf{Method} & \textbf{N Sites} & \textbf{FedMAP} & \textbf{Baseline} & \textbf{Median Diff.} & \textbf{95\% CI} & \textbf{p-value} & \textbf{Prop. Improved}\\
\midrule
FedAvg & \num{47} & \num{0.773} & \num{0.803} & \num{-0.008} & [\num{-0.041}, \num{0.000}] & \num{1.000} & \num{0.36}\\
pFedGP & \num{47} & \num{0.773} & \num{0.758} & +\num{0.022} & [\num{-0.000}, \num{0.045}] & \num{0.243} & \num{0.62}\\
pFedME & \num{47} & \num{0.773} & \num{0.778} & +\num{0.017} & [\num{0.000}, \num{0.038}] & \num{0.165} & \num{0.60}\\
pFedBayes & \num{47} & \num{0.773} & \num{0.722} & +\num{0.049} & [\num{0.015}, \num{0.073}] & \num{0.002} & \num{0.68}\\
pFedFDA & \num{47} & \num{0.773} & \num{0.750} & +\num{0.034} & [\num{0.000}, \num{0.085}] & \num{0.105} & \num{0.62}\\
\bottomrule
\end{tabular}
\footnotetext{Median performance metrics across healthcare sites. CI, confidence interval; Prop., Proportion.}
\end{table}

\clearpage
\subsection{Inference-only performance}
\begin{table}[h]
\caption{Inference performance on smallest sites.}\label{tab:group_c_results}%
\begin{tabular}{@{}lllll@{}}
\toprule
Dataset & Method & C-index & Brier score \\
\midrule
\multirow{6}{*}{\rotatebox{90}{CPRD}} & FedAvg & 0.611 ± 0.153 & 0.095 ± 0.031 &  \\
 & pFedGP & 0.480 ± 0.226 & 0.099 ± 0.033 &  \\
 & pFedME & 0.629 ± 0.155 & 0.094 ± 0.030 &  \\
 & pFedBayes & 0.626 ± 0.166 & 0.087 ± 0.029 &  \\
 & pFedFDA & 0.481 ± 0.101 & 0.098 ± 0.034 &  \\
 & FedMAP & 0.637 ± 0.162 & 0.086 ± 0.029 &  \\
\midrule
Dataset & Method & Sensitivity & Specificity & Balanced Accuracy \\
\midrule
\multirow{6}{*}{\rotatebox{90}{INTERVAL}} & FedAvg & 0.737 & 0.759 & 0.748 \\
 & pFedGP & 0.263 & 0.950 & 0.607 \\
 & pFedME & 0.895 & 0.340 & 0.618 \\
 & pFedBayes & 0.474 & 0.904 & 0.689 \\
 & pFedFDA & 0.684 & 0.738 & 0.711 \\
 & FedMAP & 0.737 & 0.748 & 0.743 \\
\midrule
\multirow{6}{*}{\rotatebox{90}{eICU}} & FedAvg & 0.269 ± 0.335 & 0.968 ± 0.036 & 0.619 ± 0.163 \\
 & pFedGP & 0.000 ± 0.000 & 0.979 ± 0.144 & 0.490 ± 0.072 \\
 & pFedME & 0.630 ± 0.375 & 0.696 ± 0.174 & 0.673 ± 0.179 \\
 & pFedBayes & 0.478 ± 0.384 & 0.765 ± 0.157 & 0.632 ± 0.189 \\
 & pFedFDA & 0.121 ± 0.217 & 0.953 ± 0.146 & 0.537 ± 0.132 \\
 & FedMAP & 0.529 ± 0.384 & 0.806 ± 0.159 & 0.668 ± 0.212 \\
\bottomrule
\end{tabular}
\footnotetext{Results reported as mean ± standard deviation per site. Bold values indicate best performance per dataset.}
\footnotetext[1]{INTERVAL has only one inference site.}
\end{table}

\clearpage
\subsection{Linear regression of relationship between individual site performance and FL benefit}
\begin{figure}[hbp]
    \centering
    \begin{subfigure}{0.45\textwidth}
        \centering
        \includegraphics[width=\linewidth]{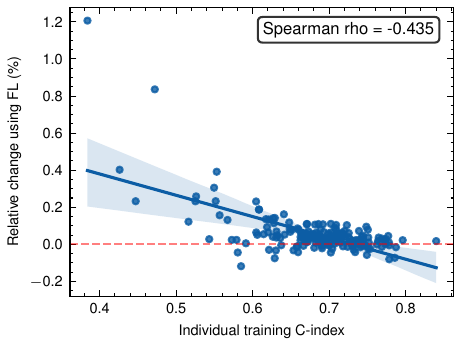}
        \caption{}
    \end{subfigure}%
    \begin{subfigure}{0.45\textwidth}
        \centering
        \includegraphics[width=\linewidth]{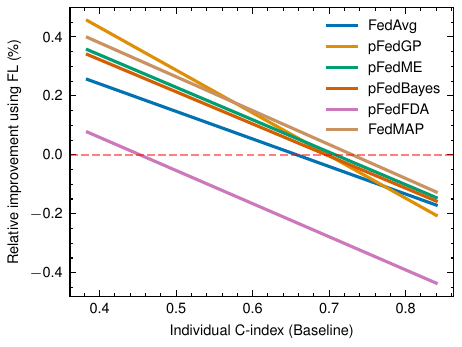}
        \caption{}
    \end{subfigure}\\
    \begin{subfigure}{0.45\textwidth}
        \centering
        \includegraphics[width=\linewidth]{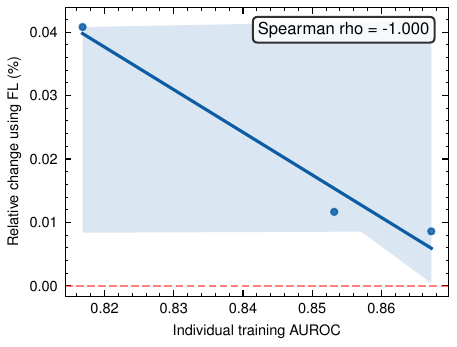}
        \caption{}
    \end{subfigure}%
    \begin{subfigure}{0.45\textwidth}
        \centering
        \includegraphics[width=\linewidth]{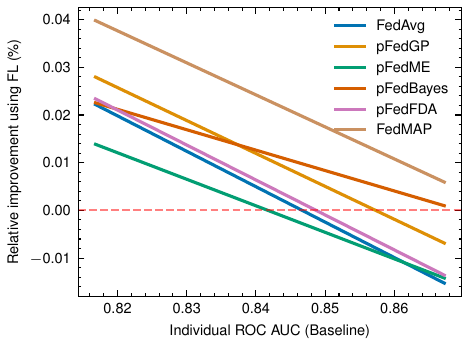}
        \caption{}
    \end{subfigure}\\
    \begin{subfigure}{0.45\textwidth}
        \centering
        \includegraphics[width=\linewidth]{img/roc_reg/eICU_auroc_reg_single_FedMAP.pdf}
        \caption{}
    \end{subfigure}%
    \begin{subfigure}{0.45\textwidth}
        \centering
        \includegraphics[width=\linewidth]{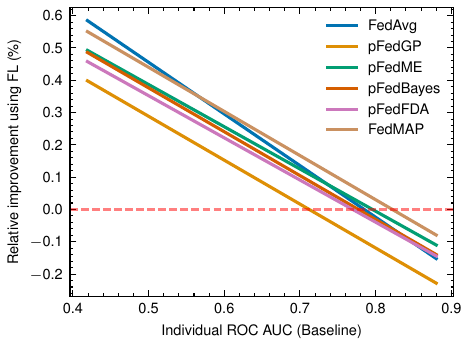}
        \caption{}
    \end{subfigure}\\

    \caption{\textbf{FL performance gains versus baseline individual site performance.}\\
    \textbf{a,c,e,} Performance from individual training on the x-axis and relative improvement using FedMAP FL on the y-axis. Solid line shows a linear regression fit with \SI{95}{\percent} confidence interval (shaded region).
    \textbf{b,d,f,} Similar to the left column but showing multiple FL methods and only showing the linear regression fit without the corresponding data or confidence interval.
    In all plots, the dashed line indicates zero improvement.
    \textbf{a,b,} CPRD.
    \textbf{c,d,} INTERVAL.
    \textbf{e,f,} eICU.
    }
    \label{fig:supp_regression}
\end{figure}

\end{appendices}
\end{document}